\documentclass[11pt]{article}
\usepackage{natbib}
\usepackage[figuresright]{rotating}
\usepackage{floatpag}
\rotfloatpagestyle{empty}  

\usepackage{amsmath}
\usepackage{amssymb}
\usepackage{verbatim}
\usepackage{hyperref}
\usepackage{amsthm, amsfonts}
\usepackage{graphicx}
\usepackage{tikz}
\usepackage{color}
\usepackage{varwidth}
\usepackage[english]{babel}
\usepackage{listings}
\usepackage{authblk}
\usepackage{blkarray}
\usepackage{times}
\usepackage[left=1.4in,right=1.4in,top=1.5in,bottom=1.5in,]{geometry}
\usepackage[width=\linewidth, font=small]{caption}

\usetikzlibrary{spy}

\usepackage{fancyhdr}
\fancyheadoffset{-0.0675\textwidth} 

\usepackage{rotating}
\usepackage{url}
\usepackage{bbm}
\usepackage{pgfplots}

\usepackage{mathtools}

\usepackage[super]{nth}

\newtheorem{theorem}{Theorem}

\newtheorem{lemma}[theorem]{Lemma}
\newtheorem{proposition}[theorem]{Proposition}
\theoremstyle{definition}

\newtheorem{example}[theorem]{Example}

\usetikzlibrary{tikzmark,calc,,arrows,shapes,decorations.pathreplacing}
\tikzset{every picture/.style={remember picture}}
\usetikzlibrary{fit,shapes.misc}
\usetikzlibrary{positioning,backgrounds}
\def\layersep{2.5cm}

\definecolor{color1}{RGB}{200,240,240}
\definecolor{color2}{RGB}{240,240,200}

\newcommand{\bo}{\boldsymbol{0}}
\newcommand{\bin}{\operatorname{bin}}
\newcommand{\Ecal}{\mathcal{E}}
\newcommand{\F}{\operatorname{F}}
\newcommand{\dec}{\operatorname{dec}}

\newcommand{\sgn}{\operatorname{sgn}}

\pgfplotsset{width=10cm,compat=1.14}

\usetikzlibrary{positioning}
\usetikzlibrary{arrows, decorations.markings}
\input{arrowsnew}

\definecolor{bl}{RGB}{20,20,150}
\definecolor{gr}{RGB}{20,180,20}
\definecolor{ic}{RGB}{180,20,20}
\definecolor{oc}{RGB}{20,180,20}

\usepackage{tcolorbox}
\makeatletter
\newcommand{\mybox}[1]{%
	\setbox0=\hbox{#1\!\!}%
	\setlength{\@tempdima}{\dimexpr\wd0+13pt}%
	\begin{tcolorbox}[colframe=black, boxrule=1pt, arc=4pt,
		left=6pt,right=6pt, top=2pt,bottom=2pt,boxsep=0pt,width=\@tempdima]
		#1
	\end{tcolorbox}
}	
	\tikzstyle{vecArrow} = [thick, decoration={markings,mark=at position
		1 with {\arrow[semithick]{open triangle 60}}},
	double distance=2.4pt, shorten >= 5.5pt,
	preaction = {decorate},
	postaction = {draw,line width=2.4pt, white,shorten >= 4.5pt}]
	
	\tikzstyle{vecArrowGray} = [thick, decoration={markings,mark=at position
		1 with {\arrow[semithick]{open triangle 60}}},
	double distance=2.4pt, shorten >= 5.5pt,
	preaction = {decorate},
	postaction = {draw,line width=2.4pt, black!10,shorten >= 4.5pt}]

\newcommand{\unitlayerr}[5]{ 
    \setlength\fboxsep{1pt}
    \myboxx{\tikzset{node distance=.4cm, auto}
        \begin{tikzpicture}[scale=0.9, every node/.style={transform shape}]
        \tikzstyle{neuron}=[circle, line width=.75pt, draw=black, inner sep=.025cm, minimum size = .65cm, fill=#5!05]
        
        \foreach \name / \i in {1,...,#2}
        \node[neuron,draw=#5!05] (I-\name) at (\i,0) {};
        
        \node (I-dots) [node distance = .8cm, right of = I-#2] {};
        \node[neuron,draw=#5!05] (I-end)  [node distance = .8cm, right of = I-dots] {};
        \end{tikzpicture}
    }}

\newcommand{\labeledlayer}[5]{ 
    \setlength\fboxsep{1pt}
    \mybox{\tikzset{node distance=.4cm, auto}
        \begin{tikzpicture}[scale=0.9, every node/.style={transform shape}]
        \tikzstyle{neuron}=[circle, line width=.75pt, draw=black, inner sep=.025cm, minimum size = .65cm, fill=#5!15]
        
        \foreach \name / \i in {1,...,#2}
        \node[neuron] (I-\name) at (\i,0) {$#1_{#4\i}$};
        
        \node (I-dots) [node distance = .8cm, right of = I-#2] {$\cdots$};
        \node[neuron] (I-end)  [node distance = .8cm, right of = I-dots] {$#1_{#4 #3}$};        
        \end{tikzpicture}
    }}

\newcommand{\unitlayer}[5]{ 
	\setlength\fboxsep{1pt}
	\mybox{\tikzset{node distance=.4cm, auto}
		\begin{tikzpicture}[scale=0.9, every node/.style={transform shape}]
		\tikzstyle{neuron}=[circle, line width=.75pt, draw=black, inner sep=.025cm, minimum size = .65cm, fill=#5!15]
		
		\foreach \name / \i in {1,...,#2}
		\node[neuron] (I-\name) at (\i,0) {$#1_{#4\i}$};
		
		\node (I-dots) [node distance = .8cm, right of = I-#2] {$\cdots$};
		\node[neuron] (I-end)  [node distance = .8cm, right of = I-dots] {$#1_{#4 #3}$};		
		\end{tikzpicture}
	}}

\newcommand{\myboxx}[1]{%
    \setbox0=\hbox{#1\!\!}%
    \setlength{\@tempdima}{\dimexpr\wd0+13pt}%
    \begin{tcolorbox}[colframe=black, boxrule=1pt, arc=4pt,
        left=6pt,right=6pt, top=70pt,bottom=10pt,boxsep=0pt,width=\@tempdima]
        #1
    \end{tcolorbox}
}

\tikzset{
    add Rectangle/.style={
        alias=tempname,
        append after command={
            ; \begin{pgfonlayer}{background}
        \node [
            fill=green!20,
                draw=yellow!50!black,
                        fit={(tempname) ($(tempname.north)+(0,-#1)$)},
                        inner sep=0pt] {};
        \end{pgfonlayer} \path
        }
    },
    add rectangle/.default=2cm
}
\tikzset{
    add rectangle/.style={
        alias=tempname,
        append after command={
            ; \begin{pgfonlayer}{background}
        \node [
            fill=blue!15,
                draw=white!10!black,
                        fit={(tempname) ($(tempname.north)+(0,-#1)$)},
                        inner sep=0pt] {};
        \end{pgfonlayer} \path
        }
    },
    add rectangle/.default=2cm
}
\tikzset{
    add RRectangle/.style={
        alias=tempname,
        append after command={
            ; \begin{pgfonlayer}{background}
        \node [
            fill=yellow!20,
                draw=white!10!black,
                        fit={(tempname) ($(tempname.north)+(0,-#1)$)},
                        inner sep=0pt] {};
        \end{pgfonlayer} \path
        }
    },
    add rectangle/.default=2cm
}
\tikzset{
    add RRectanglee/.style={
        alias=tempname,
        append after command={
            ; \begin{pgfonlayer}{background}
        \node [
            fill=red!20,
                draw=white!10!black,
                        fit={(tempname) ($(tempname.north)+(0,-#1)$)},
                        inner sep=0pt] {};
        \end{pgfonlayer} \path
        }
    },
    add rectangle/.default=2cm
}
\tikzset{
    add RECT/.style={
        alias=tempname,
        append after command={
            ; \begin{pgfonlayer}{background}
        \node [
                draw=white!10!black,
                        fit={(tempname) ($(tempname.north)+(0,-#1)$)},
                        inner sep=0pt] {};
        \end{pgfonlayer} \path
        }
    },
    add rectangle/.default=2cm
}
\tikzset{
    add RECT2/.style={
        alias=tempname,
        append after command={
            ; \begin{pgfonlayer}{background}
        \node [
                draw=yellow!100!white,
                        fit={(tempname) ($(tempname.north)+(0,-#1)$)},
                        inner sep=0pt] {};
        \end{pgfonlayer} \path
        }
    },
    add rectangle/.default=2cm
}

\def\a{\alpha }  \def\d{\delta } \def\D{\Delta } \def\e{\epsilon } \def\g{\gamma }     \def\th{\theta }  \def\r{\rho} \def\o{\omega}   \def\t{\tau}

       \def\sM{{\cal M}}       \def\s{\sigma }     \def\sE{{\cal E}} \def\sS{{\cal S}} \def\sN{{\cal N}}
\def\sO{{\cal O}}

  \def\N{{\mathbb N}}  \def\R{{\mathbb R}}

\newcommand{\bc}{\begin{center}}
\newcommand{\ec}{\end{center}}

\newcommand{\bq}{\begin{equation}}
\newcommand{\eq}{\end{equation}}

\newcommand{\bpm}{\begin{pmatrix}}
\newcommand{\epm}{\end{pmatrix}}

\newcommand{\bfig}{\begin{figure}}
\newcommand{\efig}{\end{figure}}

\newcommand{\btab}{\begin{table}}
\newcommand{\etab}{\end{table}}

\newcommand{\btbl}{\begin{tabular}}
\newcommand{\etbl}{\end{tabular}}

\newcommand{\benum}{\begin{enumerate}}
\newcommand{\eenum}{\end{enumerate}}

\newcommand{\bite}{\begin{itemize}}
\newcommand{\eite}{\end{itemize}}

\begin{document}

\title{\Large\bf Stochastic Feedforward Neural Networks:\\ Universal Approximation}

\author[1]{{\normalsize\bf Thomas Merkh}\thanks{tmerkh@math.ucla.edu}}
\author[1,2,3]{{\normalsize\bf Guido Mont\'ufar}\thanks{montufar@math.ucla.edu}}

\affil[1]{\small Department of Mathematics, University of California Los Angeles, CA 90095}
\affil[2]{\small Department of Statistics, University of California Los Angeles, CA 90095}
\affil[3]{\small Max Planck Institute for Mathematics in the Sciences, 04103 Leipzig, Germany}
\date{\small\today}
\maketitle

\thispagestyle{empty}
\vspace*{-.5cm}
\abstract{%
\noindent
In this work we take a look at the universal approximation question for stochastic feedforward neural networks. 
In contrast to deterministic neural networks, which represent mappings from a set of inputs to a set of outputs, stochastic neural networks represent mappings from a set of inputs to a set of probability distributions over the set of outputs. In particular, even if the sets of inputs and outputs are finite, the class of stochastic mappings in question is not finite. 
Moreover, while for a deterministic function the values of all output variables can be computed independently of each other given the values of the inputs, in the stochastic setting the values of the output variables may need to be correlated, which requires that their values are computed jointly. A prominent class of stochastic feedforward networks which has played a key role in the resurgence of deep learning are deep belief networks. The representational power of these networks has been studied mainly in the generative setting, as models of probability distributions without an input, or in the discriminative setting for the special case of deterministic mappings. 
We study the representational power of deep sigmoid belief networks in terms of compositions of linear transformations of probability distributions, Markov kernels, that can be expressed by the layers of the network. We investigate different types of shallow and deep architectures, and the minimal number of layers and units per layer that are sufficient and necessary in order for the network to be able to approximate any given stochastic mapping from the set of inputs to the set of outputs arbitrarily well. The discussion builds on notions of probability sharing and mixtures of product distributions, focusing on the case of binary variables and conditional probabilities given by the sigmoid of an affine map. 
After reviewing existing results, we present a detailed analysis of shallow networks and a unified analysis for a variety of deep networks. Most of the results were previously unpublished or are new. 
}\smallskip


\newpage 
\thispagestyle{empty}
\tableofcontents
\newpage 

\section{Introduction}
\label{section:one}

Obtaining detailed comparisons between deep vs.~shallow networks remains a topic of theoretical and practical importance as deep learning continues to grow in popularity. 
The successes of deep networks exhibited in many recent applications has sparked much interest in such comparisons~\citep{%
    Larochelle:2007:EED:1273496.1273556, 
	Bengio-2009, 
	delalleau2011shallow,
	pascanu2013number, 
	NIPS2014_5422,
	mhaskar2016deep, 
	pmlr-v49-eldan16, 
	poggio2017and, 
    NIPS2017_7203, 	
    pmlr-v70-raghu17a, 
	YAROTSKY2017103, 
    pmlr-v75-yarotsky18a, 
	doi:10.1137/18M118709X,
	Gribonval2019ApproximationSO}, 
and in developing theory for deep architectures. 
Despite the acclaim, guidelines for choosing the most appropriate model for a given problem have remained elusive. One approach to obtaining such guidance is to analyze the representational and approximation capabilities of different types of architectures. 
The representational power of neural networks poses a number of interesting and important questions, even if it might not capture other important and complex aspects that impact the performance in practice. 
In particular, we note that the choice of network architecture defines a particular parametrization of the representable functions, which in turn has an effect on the shape of the parameter optimization landscape. 

This work examines one aspect of this subject matter; namely, how do deep vs.~shallow stochastic feedforward networks compare in terms of the number of computational units and parameters that are sufficient in order to approximate a target stochastic function to a given accuracy? 
In contrast to deterministic neural networks, which represent mappings from a set of inputs to a set of outputs, stochastic neural networks represent mappings from a set of inputs to a set of probability distributions over the set of outputs. As so, stochastic networks can be used to model the probability distribution of a given set of training examples. This type of problem, which is an instance of parametric density estimation, is a core problem 
in statistics and machine learning. When trained on a set of unlabeled examples, a stochastic network can learn to map an internal hidden variable to new examples which follow a similar probability distribution as the training examples. They can also be trained to generate examples which follow probability distributions conditioned on given inputs. For instance, the input might specify a value ``\texttt{cat}'' or ``\texttt{dog}'', and the outputs could be images of the corresponding animals.
Generative modeling is a very active area of research in contemporary machine learning. In recent years, a particularly popular approach to generative modeling is the generative adversarial network \citep{goodfellow14} and its many variants. The distinguishing property of this approach is that the training loss is formulated in terms of the ability of a discriminator to tell apart the generated examples and the training examples. Aside from utilizing this particular type of loss, these models are implemented in the same general way, as a sequence of mappings that take an internal source to values in a desired domain. The distinguishing property of stochastic neural networks is that each layer can implement randomness. 
Learning stochastic feedforward networks has been an important topic of research for years  \citep{Neal90learningstochastic,Ngiam2011LearningDE,NIPS2013_5026,raiko2014techniques,lee2017simplified}. 
Stochastic neural networks have found applications not only as generative models, but also in unsupervised feature learning \citep{HintonSalakhutdinov2006b,4270182}, 
semantic hashing \citep{Salakhutdinov:2009:SH:1558385.1558446}, 
and natural language understanding \citep{6737243}, among others. 
Unsupervised training with stochastic networks can be used as a parameter initialization strategy for subsequent supervised learning, which was a key technique in the rise of deep learning in the years 2000
\citep{%
hinton2006fast,
NIPS2006_3048,
Bengio-2009}. 

We study the representational power of stochastic feedforward networks from a class that is known as Bayesian sigmoid belief networks \citep{NEAL199271}. These are special types of directed graphical models, also known as Bayesian networks \citep[][]{lauritzen1996graphical,Pearl:1988:PRI:52121}. 
We consider a spectrum of architectures (network topologies) in relation to universal approximation properties. 
When viewing networks as approximators of elements from a specific class, they can be quantified and compared by measures such as the worst-case error for the class. 
If a sufficiently large network is capable of approximating all desired elements with arbitrary accuracy, it can be regarded as a \textit{universal approximator}. 
The question of whether or not a certain network architecture is capable of universal approximation, and if so, how many computational units and trainable parameters suffice, has been studied for a variety of stochastic networks~\citep[see, e.g.,][]{%
	sutskever2008deep,
	le2010deep,
	10.1007/978-3-642-24412-4_3, 
	montufar2011refinements,
	montufar2011expressive, 
	montufar2014deep,
	montufar2015geometry, 
	montufar2014universal, 
	montufar2015universal,
	montufar2017hierarchical,
	73459}. 

Most of the existing works on the representational power of stochastic feedforward networks focus on the generative setting with no inputs, modeling a single probability distribution over the outputs, or the discriminative setting modeling a deterministic mapping from inputs to outputs. 
Models of stochastic functions, which are also referred to as Markov kernels or conditional probability distributions, are more complex than models of probability distributions. 
Rather than a single probability distribution, they need to approximate a probability distribution for each possible input. 
Universal approximation in this context inherently requires more complexity as compared to generative models with no inputs. 
There is also a wider variety of possible network architectures, each with virtually no guidance on how one compares to another. 
Nonetheless, as we will see, the question of universal approximation of Markov kernels can be addressed using similar tools as previously developed for studying universal approximation of probability distributions with deep belief networks~\citep{%
	sutskever2008deep,
	le2010deep,
	montufar2011refinements,
	montufar2014universal}. 
We will also draw on unpublished studies of shallow stochastic feedforward networks~\citep{montufar2015universal}. 

The overall idea of universal approximation that we consider here is as follows. For each possible input ${\bf x} \in \{0,1\}^d$ to the network, there will be a target conditional distribution $p(\cdot|{\bf x})$ over the outputs ${\bf y} \in \{0,1\}^s$ which the network attempts to learn. 
Note that while there is only a finite number $(2^s)^{2^d}$ of deterministic mappings from inputs to outputs, there is a continuum $(\Delta_{\{0,1\}^s})^{2^d}$ of mappings from inputs to probability distributions over outputs, where $\Delta_{\{0,1\}^s}$ is the $(2^s-1)$-dimensional simplex of probability distributions over $\{0,1\}^s$. 
As the number of hidden units of the network grows, the network gains the ability to better approximate the target conditional distributions. At a certain threshold, the model will have sufficient complexity to approximate each conditional distribution with arbitrarily high precision. The precise threshold is unknown except in special cases, and thus upper bounds for universal approximation are generally used to quantify a network's representational capacity. 
Since feedforward networks operate sequentially, each layer can be seen as a module that is able to implement certain operations sharing or diffusing the probability mass away from the distribution at the previous layer and toward the target distribution. This is referred to as \textit{probability mass sharing}. Depending on the size of the layers, the types of possible operations varies. The composition of operations layer by layer is a key difference between deep and shallow networks. 

We prove sufficiency bounds for universal approximation with shallow networks and with a spectrum of deep networks. The proof methods for the deep and shallow cases differ in important ways due to the compositional nature of deep architectures. This is especially so when restrictions are imposed on the width of the hidden layers.  We extend the ideas put forth by~\cite{sutskever2008deep,le2010deep,montufar2011refinements}, where universal approximation bounds were proven for deep belief networks. 
Our main results can be stated as follows. 
\begin{itemize}
\item 
\textit{A shallow sigmoid stochastic feedforward network with $d$ binary inputs, $s$ binary outputs, and a hidden layer of width $2^{d}(2^{s-1}-1)$ is a universal approximator of Markov kernels.} 

\item
\textit{There exists a spectrum of deep sigmoid stochastic feedforward networks with $d$ binary inputs, $s$ binary outputs, and 
$2^{d-j} (2^{s-b} + 2^b - 1 )$
hidden layers of width $2^j(s+d-j)$ that are universal approximators of Markov kernels.  Here $b \sim \log_2(s)$, and the overall shape of each network is controlled by $j \in \{0,1,\dots, d\}$. Moreover, each of these networks can be implemented with a minimum of  $2^d(2^s-1)$ trainable parameters.} 

\item \textit{For the networks in the previous item, if both the trainable and non-trainable parameters are restricted to have absolute values at most $\alpha$, the approximation error for any target kernel can be bounded in infinity norm by $1-\sigma(\alpha/2(d+s))^N +2\sigma(-\alpha/2(d+s))$, where $N$ is the total number of units of the network and $\sigma$ is the standard logistic sigmoid function. 
} 
\end{itemize}

\medskip 

The work is organized as follows. 
Section~\ref{sec:previous} discusses previous works for context. 
Section~\ref{sec:settings} discusses preliminary notions and fixes notation. 
Section~\ref{sec:shallow} presents an analysis of shallow networks. 
The proofs are contained in Section~\ref{sec:shallow_proofs}. 
Section~\ref{sec:deep} presents the main results, describing universal approximation with a spectrum of deep networks and approximation with bounded weights. 
The proofs of these results are contained in Section~\ref{sec:deep_proofs}. 
Section~\ref{sec:lowerbounds} discusses the lower bounds for universal approximation. 
Afterward, a brief comparison between architectures and numerical experiments are discussed in Section~\ref{sec:examples}. Last, Section~\ref{sec:conclusion} offers a conclusion and avenues for future research.

\section{Overview of Previous Works and Results}
\label{sec:previous}

The universal approximation property has been studied in a variety of contexts in the past. 
The seminal work of \citet{Cybenko1989} and \citet{hornik1989multilayer} showed that deterministic multilayer feedforward networks with at least one sufficiently large hidden layer are universal approximators over a certain class of Borel measurable functions. 
An overview on universal approximation for deterministic networks was provided by \citet{SCARSELLI199815}. 
The case of stochastic functions is not covered by this analysis, and was studied later. 
Soon after \citet{hinton2006fast} introduced a practical technique for training deep architectures, universal approximation for deep belief networks (DBNs) was shown by \citet{sutskever2008deep}. 
They found that a DBN consisting of $3(2^s-1)+1$ layers of width $s+1$ is sufficient for approximating any distribution $p \in \D_s$ arbitrarily well. 
This sufficiency bound was improved upon twice, first by \citet{le2010deep}, then by \citet{montufar2011refinements}. The former introduced the idea of using Gray codes to overlap probability sharing steps, thereby reducing the number of layers down to $\sim \frac{2^s}{s},$ each having width $s$.  The latter further reduced the number of layers to $\frac{2^{s-1}}{s-b}, \; b \sim \log_2(s)$ by improving previous results on the representational capacity of restricted Boltzmann machines (RBMs) and probability sharing theorems. 
It is interesting to note that still further improvements have been made on the representational capabilities of RBMs \citep{montufar2017hierarchical}, but it remains unclear whether or not universal approximation bounds for DBNs can benefit from such improvements. For a recent review of results on RBMs, see \citep{10.1007/978-3-319-97798-0_4}. 

Several stochastic networks in addition to DBNs have been shown to be universal approximators.  The undirected counterpart to DBNs called a deep Boltzmann machine (DBM) was proven to be a universal approximator even if the hidden layers are restricted to have at most the same width as the output layer \citep{montufar2014deep}.  Here it was shown that DBMs could be analyzed similarly to feedforward networks under certain parameter regimes. 
This result verifies the intuition that undirected graphical models are in some well defined sense at least as powerful as their directed counterparts. 
For shallow stochastic networks universal approximation bounds have been obtained for feedforward networks which will be discussed next, and undirected networks called conditional restricted Boltzmann machines (CRBMs) \citep{montufar2015geometry}.  Both such architectures are capable of universal approximation and have similar sufficiency bounds. 

We note that not every network architecture is capable of universal approximation.  For example, it is known that for an RBM, DBN, or DBM to be a universal approximator of distributions over $\{0,1\}^s$, the hidden layer immediately before the visible layer must have at least $s-1$ units~\citep{montufar2014deep}. In fact, if $s$ is odd, at least $s$ are needed \citep{doi:10.1137/140957081}. In addition, necessary bounds for universal approximation exist for all of the previously mentioned architectures, though such bounds are generally harder to refine.  Except for very small models, there exists a gap between the known necessary bounds and the sufficiency bounds.  Last, it was recently shown that a class of deterministic feedforward networks with hidden layer width at most equal to the input dimension are unable to capture a class of functions with bounded level sets \citep{johnson2018deep}. Such discoveries exemplify the importance of analyzing the approximation capabilities of different network architectures. 

As already mentioned in the introduction, the representational power of discriminative models has been previously studied. In particular, the representation of deterministic functions from $\{0,1\}^d \to \{0,1\}$, known as Boolean functions, by logical circuits or threshold gates has been studied for many years. 
\citet{6771698} showed that almost all $d$-input Boolean functions require a logic circuit of size at least $(1 - o(1))2^d/d$. 
\citet{lupanov19562} showed that every $d$-input Boolean function can be expressed by a logic circuit of size at most $(1 + o(1))2^d/d$. 
Other works on the representation of Boolean functions include  \citet{brunato2015stochastic,huang2006can,muroga1971threshold,neciporuk1964synthesis,wenzel2000hyperplane}. 
A particularly interesting result by \citet{Hastad1991} shows that, when the depth of a threshold circuit is restricted, some Boolean functions require exponentially more units to be represented. 

\citet{rojas2003networks} showed that a sufficiently deep stack of perceptrons where each layer is connected to the input and feeds forward a single bit of information is capable of learning any $d$-input Boolean function. 
The equivalent network without skip connections to the input would be a network of width $d+1$. 
In that work it is pointed out that there is a direct trade-off between the width and depth of the network, and this idea will surface again in the analysis of deep networks that follows. 
\citet{le2010deep} showed that a sigmoid belief network with $2^{d-1}+1$ layers of width $d$ is sufficient for representing any deterministic function $f\colon \{0,1\}^d \to \{0,1\}$. 
This was done by showing that the parameters of one layer can be chosen to map a single vector to some fixed ${\bf h}_0 \in \{0,1\}^d$. Then considering two classes of vectors, those for which $f({\bf h}) = 0$ and those for which $f({\bf h}) = 1$, one may choose to map the smaller of the two classes of vectors to ${\bf h}_0$. This can be done in $2^{d-1}$ layers or less, and then the last layer can correctly label the inputs depending on whether or not the network mapped them to ${\bf h}_0$ or not. This process differs from the following in that these networks are not learning multivariate conditional distributions for each input, but rather labeling each input $0$ or $1$. 
While for a deterministic function the values of all output variables can be computed independently of each other given the values of the inputs, in the stochastic setting the values of the output variables may be correlated, which requires that their values are computed jointly.

\section{Markov Kernels and Stochastic Networks}
\label{sec:settings}

\subsection*{Binary Probability Distributions and Markov Kernels}

Let $s \in \N$ and consider the set of vectors $\{0,1\}^s$ of cardinality $2^s$.
A probability distribution over the set $\{0,1\}^s$ is a vector $p\in\mathbb{R}^{2^s}$ with non-negative entries $p_i$, $i\in\{1,\ldots, 2^s\}$ that add to one. The entries correspond to the probabilities $p({\bf y})$ that this distribution assigns to each ${\bf y}\in\{0,1\}^s$. 
The set of all such probability distributions is the set 
\begin{equation}
	\D_s := \Big\{ p \in \R^{2^s} \; \Big| \; \sum_{i=1}^{2^s} p_i = 1, \;\; p_i \geq 0 \text{ for } i = 1,2,\ldots,2^s \Big\} . 
\end{equation}
This set is a simplex of dimension $2^s-1$. 
The vertices of $\D_s$ are point distributions which assign full probability mass to a single ${\bf y} \in \{0,1\}^s$ and none on $\{0,1\}^s \setminus {\bf y}$.  Such distributions are denoted $\d_{\bf y}$ and are sometimes referred to as \textit{deterministic} because there is no uncertainty under them. 
The support of a distribution $p \in \D_s$ is denoted by $\operatorname{supp}(p)$ and is the set of the vectors in $\{0,1\}^s$ on which $p$ assigns nonzero probability. The support set of a deterministic distribution is a singleton. 
A probability model $\sM$ is just a subset of $\D_s$. 
If a model $\sM \subseteq \D_s$ satisfies $\overline{\sM} = \D_s$, then $\sM$ is said to have the universal approximation property. Here $\overline{\sM}$ refers to the closure of $\sM$ in the Euclidean topology.

A stochastic map or Markov kernel with input space $\{0,1\}^d$ and output space $\{0,1\}^s$ is a map 
$P \colon \{0,1\}^d \to \D_s$. Such a Markov kernel can be see as a $2^d \times 2^s$ matrix with non-negative entries and with rows that sum to $1$. The $i$-th row is the probability distribution over $\{0,1\}^s$ corresponding to the $i$-th input. 
The set of all Markov kernels is written as  
\begin{equation}
	\D_{d,s} := \Big\{ P \in \R^{2^d \times 2^s} \, : \, P_{ij} \geq 0, \; \sum_{j=1}^{2^s} P_{ij} = 1 \text{ for all } i = 1,2,\dots, 2^d \Big\}.
\end{equation}
One can see that $\D_{d,s}$ is the $2^d$-fold Cartesian product of $\D_s$, and thus is a polytope of dimension $2^d(2^s-1)$. 
A model of Markov kernels $\sN$ is simply a subset of $\D_{d,s}$. 
If a model $\sN$ fills this polytope, meaning $\overline{\sN} = \D_{d,s}$, it is said to be a universal approximator of Markov kernels.

An important class of probability distributions are the factorizable or independent distributions,
for which the probability of observing a joint state ${\bf z}\in\{0,1\}^d$ is just the product of the probabilities of observing each state $z_j\in\{0,1\}$, $j=1,\ldots, d$ individually. 
These distributions can be written as 
\begin{align}
	p({\bf z}) 
	= \prod_{j = 1}^d p_{z_j}(z_j)
	= \prod_{j = 1}^d p_j^{z_j}(1-p_j)^{1-z_j}, \;\; \forall {\bf z} = (z_1,\dots,z_d) \in \{0,1\}^d,
	\label{eq:indepmod}
\end{align}
where $p_{z_j}$ is a probability distribution over $\{0,1\}$ and $p_j = p_{z_j}(z_j=1)\in [0,1]$ is the probability of the event $z_j=1$, for $j = 1,\ldots, d$. 
We denote the set of all factorizable distributions of $d$ binary variables, of the form given above, by $\sE_d$. 
The Hamming distance $\|a - b\|_H$ between two vectors $a$ and $b$ is the number of positions where $a$ and $b$ differ.  
One notable property of $\sE_d$ is that if ${\bf x'},{\bf x''} \in \{0,1\}^d$ have Hamming distance $\|{\bf x'} - {\bf x''}\|_H = 1$, 
then \textit{any} probability distribution $p \in \D_d$ with $\operatorname{supp}(p) = \{{\bf x'},{\bf x''}\}$ is in $\sE_d$. 

Certain configurations of binary vectors will be important in our analysis. 
The set of $d$-bit binary vectors can be visualized as the vertex set of the $d$-dimensional hypercube. 
If ${\bf x'},{\bf x''} \in \{0,1\}^d$ have Hamming distance $\|{\bf x'} - {\bf x''}\|_H = 1$, they form an edge of the $d$-cube. 
For this reason, they are sometimes referred to as an \textit{edge pair} in the literature. 
A codimension $0 \leq j \leq d$ face of the $d$-cube consists of the $2^{d-j}$ vectors having the same $j$ bits in common.

\subsection*{Stochastic Feedforward Networks}

We consider stochastic networks known as Bayesian sigmoid belief networks \citep{NEAL199271}, which are Bayesian networks \citep{Pearl:1988:PRI:52121,lauritzen1996graphical} with conditional distributions taking the form of a logistic sigmoid function applied to a linear combination of the parent variables. 
Details can be found in \citep{Saul:1996:MFT:1622737.1622741} and \citep[][Section~8.1]{Bishop:2006:PRM:1162264}. 

Each unit of the network represents a random variable and edges capture dependencies between variables. 
In our discussion, all of the units of the network are binary. 
The graphs are directed and acyclic, so that the units can be arranged into a sequence of layers. 
We will focus on the case where consecutive layers are fully connected and there are no intralayer and no skip connections. 
The units in each layer are conditionally independent given the state of the units in the previous layer. 
Figure~\ref{fig:zero} shows an example of such an architecture. 

We denote the binary inputs by ${\bf x} \in \{0,1\}^d$, and the outputs by ${\bf y} \in \{0,1\}^s$. 
The network's computational units take states in $\{0,1\}$ with activation probabilities given by the sigmoid function applied to an affine transformation of the previous layer's values. 
Specifically, given a state ${\bf h}^{l-1}\in\{0,1\}^{m_{l-1}}$ of the $m_{l-1}$ units in layer $l-1$, the $j$-th unit of the $l$-th layer \textit{activates} (i.e.\ it takes state $h_j^l = 1$) with probability 
\begin{equation}
	p(h_j^l = 1 | {\bf h}^{l-1})
	= \s({\bf W}_j^l{\bf h}^{l-1} + b_j^l) 
	= \frac{1}{1 + e^{-({\bf W}_j^l{\bf h}^{l-1} +  b_j^l)}}. 
\end{equation}
Here ${\bf W}^l_{j}\in\mathbb{R}^{1\times m_{l-1}}$ is a row vector of weights and $b_j^l\in\mathbb{R}$ is a bias. 
The weights and bias of all units in layer $l$ are collected in a matrix ${\bf W}^l \in \R^{m_l} \times \R^{m_{l-1}}$ and a vector ${\bf b}^l = (b_1^l, \dots, b_{m_l}^l) \in \R^{m_l}$. 
We denote the parameters (weights and biases) of the entire network collectively by $\theta$. 
Note that the inverse of the sigmoid function $\s$ is known as the \textit{logit} function $\s^{-1}(x) = \log{(\frac{x}{1-x})} = \log({x}) - \log{(1-x)}.$ 
The units in layer $l$ are conditionally independent given the state ${\bf h}^{l-1}$ of the units in the preceding layer. The probability of observing state ${\bf h}^l = (h_1^l, \dots, h_{m_l}^l) \in \{0,1\}^{m_l}$ at layer $l$ given ${\bf h}^{l-1}$ is
\begin{equation}
p({\bf h}^l \, | \, {\bf h}^{l-1}) = \prod_{j = 1}^{m_l} \s({\bf W}^l_j{\bf h}^{l-1} + b_{j}^l)^{h_j^l} \big(1-\s({\bf W}^l_j{\bf h}^{l-1} + b_{j}^l) \big)^{1-h_j^l}. 
\label{eq:module}
\end{equation}

Given an input ${\bf x}$, the conditional distribution of all units in a network with $L+2$ layers (including input and output layers) can be written as 
\begin{equation}
p({\bf h}^1, {\bf h}^2, \dots, {\bf h}^L, {\bf y}|{\bf x}) = p({\bf y} | {\bf h}^L)p({\bf h}^L | {\bf h}^{L-1}) \cdots p({\bf h}^1| {\bf x}) . 
\label{eq:jointdistribution}
\end{equation}
By marginalizing over the hidden layers, which are all layers other than the input and output layers, one obtains the conditional distribution of the output given the input as
\begin{equation}
p({\bf y} | {\bf x}) 
= \sum_{{\bf h}^1} \cdots \sum_{{\bf h}^L} p({\bf y} | {\bf h}^L)p({\bf h}^L | {\bf h}^{L-1}) \cdots p({\bf h}^1| {\bf x}). 
\label{eq:condiprob}
\end{equation}
In particular, a network with fixed parameters represents a Markov kernel in $\D_{d,s}$. When we allow the network's parameter $\theta$ to vary arbitrarily, we obtain the set of all Markov kernels in $\Delta_{d,s}$ that are representable by the particular network architecture. 
The architecture is fully determined by the number of layers and the sizes $m_1,\ldots, m_L$ of the hidden layers. 
We call a network \textit{shallow} if $L = 1$, and \textit{deep} if $L > 1$. 

We denote $F_{d,s}\subseteq\D_{d,s}$ the set of all Markov kernels that can be represented by a network module of the form \eqref{eq:module} with an input layer of size $d$ and an output layer of size $s$, with no hidden layers. 
Networks with $L > 1$ can be seen as the composition of $L+1$ such network modules. 
We denote $F_{d,m_1,\dots,m_L,s} := F_{m_L,s} \,\circ\, F_{m_{L-1},m_L} \,\circ\, \cdots \,\circ\, F_{m_1,m_2} \,\circ\, F_{d,m_1} \subseteq \D_{d,s}$ the set of all Markov kernels of the form \eqref{eq:condiprob} representable by a network architecture with an input layer of size $d$, hidden layers of size $m_l$ for $l = 1,\dots,L$, and an output layer of size $s$. 

The general task in practice for our feedforward stochastic network is to learn a conditional distribution $p^*(\cdot | {\bf x})$ for each given input ${\bf x}$. 
In other words, when providing the network with input ${\bf x}$, the goal is to have the outputs ${\bf y}$ distributed according to some target distribution $p^*(\cdot | {\bf x})$. We will be interested in the question of which network architectures have the universal approximation property, meaning that they are capable of representing \textit{any} Markov kernel in $\D_{d,s}$ with arbitrarily high accuracy.

Our analysis builds on previous works discussing closely related types of stochastic networks. For completeness, we now provide the definition of those networks. 
A restricted Boltzmann machine (RBM) with $m$ hidden and $d$ visible binary units is a probability model in $\D_d$ consisting of the distributions 
\begin{equation}
p({\bf x})  = \sum_{{\bf h}\in \{0,1\}^m } \frac{1}{Z({\bf W},{\bf b},{\bf c})}\exp({\bf x}^\top {\bf W} {\bf h} + {\bf x}^\top {\bf b} + {\bf h}^\top {\bf c} ),\quad \forall {\bf x}\in\{0,1\}^d, 
\label{eq:RBM}
\end{equation}
where ${\bf W}\in\mathbb{R}^{d\times m}$, ${\bf b}\in\mathbb{R}^d$, ${\bf c}\in\mathbb{R}^m$ are weights and biases, and $Z$ is defined in such a way that $\sum_{{\bf x}\in\{0,1\}^d }p({\bf x}) = 1$ for any choice of the weights and biases. This is an undirected graphical model with hidden variables. 
A deep belief network (DBN) is a probability model constructed by composing a restricted Boltzmann machine and a Bayesian sigmoid belief network as described above. 
It represents probability distributions of the form 
\begin{equation}
p({\bf y}) = \sum_{{\bf x}\in\{0,1\}^d} p({\bf y} | {\bf x}) p({\bf x}), \quad \forall {\bf y}\in\{0,1\}^s, 
\end{equation}
where $p({\bf y}| {\bf x})$ is a conditional probability distribution of the form \eqref{eq:condiprob}, and $p({\bf x})$ is a probability distribution of the form \eqref{eq:RBM}.

\section{Results for Shallow Networks}
\label{sec:shallow}

\begin{figure}
	\centering
	\tikzset{node distance=1.5cm, auto}
	\begin{tikzpicture}[scale=0.9, every node/.style={transform shape}]
	\node (L-0) at (0,0) {\unitlayer{y}{2}{s}{ }{oc}}; 
	\node (L-1) at (0,2) {\unitlayer{h}{5}{m}{ }{bl}}; 
	\node (L-2) at (0,4) {\unitlayer{x}{3}{d}{ }{ic}}; 
	
	\draw[latexnew-, arrowhead=.225cm,  line width = 1pt] (L-0) to node {${\bf W}$} (L-1);
	\draw[latexnew-, arrowhead=.225cm,  line width = 1pt] (L-1) to node {${\bf V}$}  (L-2);
	
	\node (c) at (2,3) { };
	\draw[latexnew-, arrowhead=.225cm,  line width = 1pt] (L-1) to node {${\bf c}$} (c);
	
	\node (b) at (2,1){ };
	\draw[latexnew-, arrowhead=.225cm,  line width = 1pt] (L-0) to node {${\bf b}$} (b);
	
	\node (Q) at (-3.5,3) {$\F_{d,m}$};
	\node (R) at (-3.5,1) {$\F_{m,s}$};
	
    \node (o) [node distance=.8cm, below of = L-0] {Output layer};
    \node (i) [node distance=.8cm, above of = L-2] {Input layer};	
	
	\end{tikzpicture}
	\caption{Feedforward network with a layer of $d$ input units, a layer of $m$ hidden units, and a layer of $s$ output units. Weights and biases are $({\bf V},{\bf c})$ for the hidden layer, and $({\bf W},{\bf b})$ for the output layer.}
	\label{figure:structure}
\end{figure}
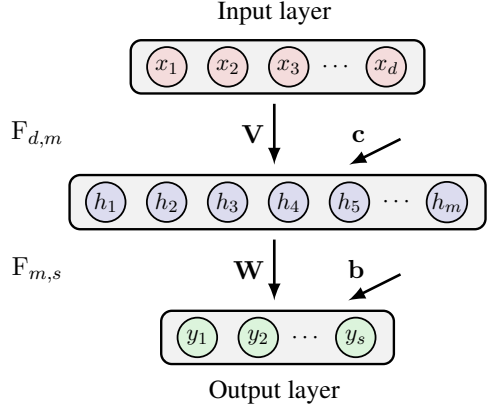

In the case of shallow networks, which have an input layer, a single hidden layer, and an output layer, as shown in Figure~\ref{figure:structure}, we are interested in the smallest size of the hidden layer which will provide for universal approximation capacity. The results in this section are collected from a technical report~\citep{montufar2015universal}, with a few adjustments. 

The shallow network $F_{d,m,s}=F_{m,s}\circ F_{d,m}$ has a total of $dm + m + sm + s$ free parameters. 
We will also consider a restricted case where the second module has fixed weights, meaning that we fix $R\in F_{m,s}$ and consider the composition $R\circ F_{d,m}$, which has only $dm + m$ free parameters. 
By comparing the number of free parameters and the dimension of $\Delta_{d,s}$, which is $2^d(2^s-1)$, it is possible to obtain (see Theorem~\ref{thm:counting} in Section~\ref{sec:lowerbounds}) the following lower bound on the minimal number of hidden units that suffices for universal approximation: 
\begin{proposition}
	\label{proposition:minimal}
	Let $d\geq 1$ and $s\geq 1$. 
	\begin{itemize}
		\item 
		If there is a $R\in \overline{\F_{m,s}}$ with $R\circ \overline{\F_{d,m}} = \Delta_{d,s}$, 
		then $m\geq 
		\frac{1}{(d+1)}2^d(2^s-1)
		$. 
		\item
		If $\overline{\F_{d,m,s}} = \Delta_{d,s}$, then $m\geq 
		\frac{1}{(s+d+1)}(2^d(2^s-1)-s)
		$. 
	\end{itemize}
\end{proposition}

In the following, we will bound the minimal number of hidden units of a universal approximator from above. 
First we consider the case where the output layer has fixed weights and biases. 
Then we consider the case where all weights and biases are free parameters. 

\subsection*{Fixed Weights in the Output Layer}

\begin{theorem}
	\label{theorem:first}
	Let $d\geq 1$ and $s\geq 1$. 
	A shallow sigmoid stochastic feedforward network with $d$ inputs, $m$ units in the hidden layer, $s$ outputs, and fixed weights and biases in the output layer is a universal approximator of Markov kernels in $\Delta_{d,s}$ whenever $m\geq 2^{d-1}(2^s-1)$. 
\end{theorem}
The theorem will be shown by constructing $R\in\overline{\F_{m,s}}$ such that $R\circ \overline{\F_{d,m}} = \Delta_{d,s}$, whenever $m\geq \frac{1}{2} 2^d(2^s-1)$. In view of the lower bound from Proposition~\ref{proposition:minimal}, this upper bound is tight at least when $d=1$.  

When there are no input units, i.e., $d=0$, we may set $\F_{0,m} =\Ecal_m$, the set of factorizable distributions of $m$ binary variables \eqref{eq:indepmod}, which has $m$ free parameters, and $\Delta_{0,s} =\Delta_s$, the set of all probability distributions over $\{0,1\}^s$. 
Theorem~\ref{theorem:first} generalizes to this case as: 
\begin{proposition}
	\label{proposition:zero}
	Let $s\geq 2$. 
	There is an $R\in\overline{\F_{m,s}}$ with $R\circ \overline{\Ecal_m}=\Delta_s$, whenever $m\geq 2^{s}-1$. 
\end{proposition}
This bound is always tight, since the network uses exactly $2^s-1$ parameters to approximate every distribution from $\Delta_{s}$ arbitrarily well. 
For $s=1$, one hidden unit is sufficient and necessary for universal approximation.

\subsection*{Trainable Weights in the Output Layer}

When we allow for trainable weights and biases in both layers, we obtain a slightly more compact bound: 
\begin{theorem}
	\label{theorem:second}
    Let $d\geq 1$ and $s\geq 2$. A shallow sigmoid stochastic feedforward network with $d$ inputs, $m$ units in the hidden layer, and $s$ outputs is a universal approximator of kernels in $\Delta_{d,s}$ whenever $m\geq 2^{d}(2^{s-1}-1)$. 
\end{theorem}

This bound on the number of hidden units is slightly smaller than the one obtained for fixed weights in the output layer. However, it always leaves a gap to the corresponding parameter counting lower bound. 

As before, we can also consider the setting where there are no input units, $d=0$, in which case the units in the hidden layer may assume an arbitrary product distribution, $\F_{0,m} =\Ecal_m$. In this case we obtain: 
\begin{proposition}
Let $s\geq1$. Then $\overline{F_{m,s}}\,\circ\, \overline{\Ecal_m} = \Delta_s$, whenever $m\geq 2^{s-1}-1$. 
\end{proposition}
For $s=1$, the bias of the output unit can be adjusted to obtain any desired distribution, and hence no hidden units are needed. For $s=2$, a single hidden unit, $m=1$, is sufficient and necessary for universal approximation, so that the bound is tight. For $s=3$, three hidden units are necessary \citep[][Proposition~3.19]{doi:10.1137/140957081}, so that the bound is also tight in this case.

\section{Proofs for Shallow Networks} 
\label{sec:shallow_proofs}

We first give an outline of the proofs, and then proceed with the analysis, first for the case of fixed weights in the output layer, and then for the case of trainable weights in the output layer. 
Our strategy for proving Theorem~\ref{theorem:first} and Theorem~\ref{theorem:second} can be summarized as follows: 

\begin{itemize}
	\item 
	First we show that the first layer of $\F_{d,m,s}$ can approximate Markov kernels arbitrarily well, 
	which fix the state of some units, depending on the input, 
	and have an arbitrary product distribution over the states of the other units. 
	The idea is illustrated in Fig.~\ref{figure:proofidea}. 
	\item 
	Then we show that the second layer can approximate deterministic kernels arbitrarily well, whose rows are copies of all point measures from $\Delta_s$, ordered in a good way with respect to the different inputs. 
	Note that the point measures are the vertices of the simplex $\Delta_s$. 
	\item 
	Finally, we show that the set of product distributions of each block of hidden units is mapped to the convex hull of the rows of the kernel represented by the second layer, which is $\Delta_s$. 
	\item 
	The output distributions of distinct sets of inputs is modeled individually by distinct blocks of hidden units and so we obtain the universal approximation of Markov kernels. 
\end{itemize}

The goal of our analysis is to construct the individual pieces of the network as compact as possible. 
Lemma~\ref{lemma:firstlayer} will provide a trick that allows us to use each block of units in the hidden layer for a pair of distinct input vectors at the same time. 
This allows us to halve the number of hidden units that would be needed if each input had an individual block of active hidden units. 
Similarly, Lemma~\ref{lemma:deterministicgeometric2} will provide a trick for producing more flexible mixture components at the output layer than simply point measures. This comes at the expense of allowing only one input per hidden block, but it allows us to nearly halve the number of hidden units per block, for a slight reduction in the total number of hidden units. 

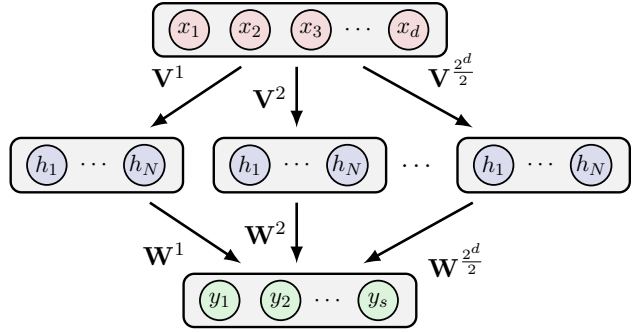
\begin{figure}
	\centering
	\tikzset{node distance=1.5cm, auto}
	\begin{tikzpicture}[scale=0.9, every node/.style={transform shape}]
	\node (L-0) at (0,0) {\unitlayer{y}{2}{s}{ }{oc}}; 
	\node (L-1a) at (-3,2) {\unitlayer{h}{1}{N}{ }{bl}}; 
	\node (L-1b) at (0,2) {\unitlayer{h}{1}{N}{ }{bl}}; 
	\node (L-1d) at (1.8,2) {$\dots$}; 
	\node (L-1c) at (3.6,2) {\unitlayer{h}{1}{N}{ }{bl}}; 
	\node (L-2) at (0,4) {\unitlayer{x}{3}{d}{ }{ic}}; 
	
	\draw[latexnew-, arrowhead=.225cm,  line width = 1pt] (L-0) to node {${\bf W}^1$} (L-1a);
	\draw[latexnew-, arrowhead=.225cm,  line width = 1pt] (L-0) to node {${\bf W}^2$} (L-1b);
	\draw[latexnew-, arrowhead=.225cm,  line width = 1pt] (L-0) to node [swap] {${\bf W}^{\frac{2^d}{2}}$} (L-1c);
	\draw[latexnew-, arrowhead=.225cm,  line width = 1pt] (L-1a) to node {${\bf V}^1$}  (L-2);
	\draw[latexnew-, arrowhead=.225cm,  line width = 1pt] (L-1b) to node {${\bf V}^2$}  (L-2);
	\draw[latexnew-, arrowhead=.225cm,  line width = 1pt] (L-1c) to node [swap] {${\bf V}^{\frac{2^d}{2}}$}  (L-2);
	
	\end{tikzpicture}
	\caption{Illustration of the construction used in our proof. 
		Each block of hidden units is active on a distinct subset of possible inputs. 
		The output layer integrates the activities of the block that was activated by the input, 
		and produces corresponding activities of the output units. }
	\label{figure:proofidea}
\end{figure}

\subsection{Fixed Weights in the Output Layer}
\label{section:proof1}

\subsubsection*{The First Layer} 

We start with the following lemma. 

\begin{lemma} 
	\label{lemma:firstlayer}
	Let ${\bf x}',{\bf x}''\in\{0,1\}^d$ differ only in one entry, and let $q',q''$ be any two distributions on $\{0,1\}$. 
	Then $F_{d,1}$ can approximate the following arbitrarily well: 
	\begin{equation*}
	p(\cdot|{\bf x}) = 
	\left\{
	\begin{array}{l l}
	q', & \text{if ${\bf x}={\bf x}'$}\\
	q'', & \text{if ${\bf x}={\bf x}''$}\\
	\delta_{0}, & \text{else}
	\end{array}
	\right.. 
	\end{equation*}
\end{lemma}
\begin{proof}
    Given the input weights and bias, ${\bf V}\in\R^{1\times d}$ and $c\in\R$, 
	for each input ${\bf x}\in\{0,1\}^d$ the output probability is given by 
	\begin{equation}
	p(z=1|{\bf x}) = \sigma({\bf V} {\bf x} + c).
	\label{eq:output}
	\end{equation} 
	Since the two vectors ${\bf x}',{\bf x}''\in\{0,1\}^d$ differ only in one entry, they form an edge pair $E$ of the $d$-dimensional unit cube. 
	Let $l\in[d]:=\{1,\ldots,d\}$ be the entry in which they differ, with $ x'_l=0$ and $x''_l=1$. 
	Since $E$ is a face of the cube, there is a supporting hyperplane of $E$. 
	This means that there are ${\bf \tilde V}\in \R^{1\times d}$ and $\tilde c\in\R$ with 
	${\bf \tilde V }{\bf x} + \tilde c =0$ if ${\bf x}\in E$, 
	and ${\bf \tilde V }{\bf x} + \tilde c < -1$ if ${\bf x}\in\{0,1\}^d\setminus E$. 
	Let $\g'= \sigma^{-1}(q'(z=1))$ and $\g''=\sigma^{-1}(q''(z=1))$. 
	We define $c= \alpha \tilde c + \g'$ and ${\bf V}=\alpha{\bf \tilde V} + (\g'' - \g')e_l^\top$. 
	Then, as $\alpha\to\infty$, 
	\begin{equation*}
	{\bf V} {\bf x} +c = 
	\left\{
	\begin{array}{l l}
	\g', & \text{if ${\bf x}={\bf x}'$}\\
	\g'', & \text{if ${\bf x}={\bf x}''$}\\
	-\infty, & \text{else}
	\end{array}
	\right.. 
	\end{equation*}  
	Plugging this into~\eqref{eq:output} proves the claim. 
\end{proof}

Given any binary vector ${\bf x}=(x_1,\ldots, x_d)\in\{0,1\}^d$, 
let $\dec({\bf x}):=\sum_{i=1}^d 2^{i-1} x_i$ be its integer representation. Using the previous lemma, we obtain the following. 

\begin{proposition}
	\label{proposition:firstla}
	Let $N\geq 1$ and $m=2^{d-1} N$. 
	For each ${\bf x}\in\{0,1\}^d$, let $p(\cdot|{\bf x})$ be an arbitrary factorizing distribution from $\Ecal_{N}$. 
	The model $\F_{d,m}$ can approximate the following kernel from $\Delta_{d,m}$ arbitrarily well: 
	\begin{align*}
	P({\bf h} | {\bf x} )  = &
	\delta_{\bo}({\bf h}^0 )   \cdots  
	\delta_{\bo}({\bf h}^{\lfloor\dec({\bf x})/2\rfloor-1}) 
	p({\bf h}^{\lfloor\dec({\bf x})/2\rfloor}|{\bf x}) \\
	&
	\times
	\delta_{\bo}({\bf h}^{\lfloor\dec({\bf x})/2\rfloor+1}) \cdots
	\delta_{\bo}({\bf h}^{2^{d-1} -1}), 
	\quad \forall {\bf h}\in\{0,1\}^m, {\bf x}\in\{0,1\}^d,  
	\end{align*}
	where ${\bf h}^i= (h_{Ni+1},\ldots, h_{N(i+1)})$ for all $i\in\{0,1,\ldots,2^{d-1}-1\}$. 
\end{proposition}

\begin{proof}
	We divide the set $\{0,1\}^d$ of all possible inputs into $2^{d-1}$ disjoint pairs with successive decimal values. 
	The $i$-th pair consists of the two vectors ${\bf x}$ with $\lfloor \dec({\bf x})/2 \rfloor = i$, for all $i\in\{0,\ldots,2^{d-1}-1 \}$. 
	The kernel $P$ has the property that, for the $i$-th input pair, 
	all output units are inactive with probability one, 
	except those with index $Ni+1, \ldots, N(i+1)$. 
	Given a joint distribution $q$ let $q_{j}$ denote the corresponding marginal distribution on the states of the $j$-th of unit. 
	By Lemma~\ref{lemma:firstlayer}, we can set 
	\begin{equation*}
	P_{Ni+j}(\cdot|{\bf x}) = 
	\left\{
	\begin{array}{l l}
	p_j(\cdot|{\bf x}) , & \text{if $\dec({\bf x})= 2i  $}\\
	p_j(\cdot|{\bf x}) , & \text{if $\dec({\bf x})= 2i+1$}\\
	\delta_{0}, & \text{else}
	\end{array}
	\right.  
	\end{equation*}
	for all $i\in\{0,\ldots,2^{d-1}-1 \}$ and $j\in\{1,\ldots, N\}$. 
\end{proof}

\subsubsection*{The Second Layer}

For the second layer we will consider deterministic kernels. 
Given a binary vector ${\bf z}$, 
let $l({\bf z}):=\lceil \log_2(\dec({\bf z})+1)\rceil$ denote the largest $j$ where $z_j=1$. 
Here we set $l(0,\ldots,0)=0$. 
Given an integer $l\in\{0,\ldots, 2^s-1 \}$, let $\bin_s(l)$ denote the $s$-bit representation of $l$; 
that is, the vector with $\dec(\bin_s(l))=l$.  Last, when applying any scalar operation to a vector, such as subtraction by a number, it should be understood as being performed pointwise to each vector element. 

\begin{lemma}
	\label{lemma:deterministicgeometric}
	Let $N=2^s-1$. The set $\F_{N,s}$ can approximate the following deterministic kernel arbitrarily well: 
	\begin{equation*}
	Q(\cdot | {\bf z}) = \delta_{\bin_s l({\bf z})}(\cdot), \quad \forall {\bf z}\in\{0,1\}^N. 
	\end{equation*}
\end{lemma}
In words, the ${\bf z}$-th row of $Q$ indicates the largest non-zero entry of the binary vector ${\bf z}$. 
For example, for $s=2$ we have $N=3$ and 
\begin{equation*}
Q = 
\begin{blockarray}{ccccc}
00 & 01 & 10 & 11 & \\
\begin{block}{(cccc) c}
1 &    &      &      & 000\\
& 1 &      &      & 001\\
&    & 1   &      & 010\\
&    & 1   &      & 011\\
&    &      &  1  & 100\\
&    &      &  1 & 101\\
&    &      &  1 & 110\\
&    &      &  1 & 111\\
\end{block}
\end{blockarray}. 
\end{equation*}

\begin{proof}[Proof of Lemma~\ref{lemma:deterministicgeometric}]
Given the input and bias weights, ${\bf W}\in\R^{s\times N}$ and ${\bf b}\in\R^{s}$, for each input ${\bf z}\in\{0,1\}^N$ the output distribution is the product distribution $p({\bf y}|{\bf z}) = \frac{1}{Z({\bf W}z + {\bf b})}\exp({\bf y}^\top ({\bf W}z + {\bf b}))$ in $\sE_s$ with parameter ${\bf W} {\bf z} + {\bf b}$. 
If $\sgn({\bf W} {\bf z} + {\bf b}) = \sgn({\bf x}-\tfrac12)$ for some ${\bf x}\in\{0,1\}^s$, 
then the product distribution with parameters $\alpha({\bf W}{\bf z}+{\bf b})$, $\alpha\to\infty$ tends to $\delta_{\bf x}$. 
We only need to show that there is a choice of ${\bf W}$ and ${\bf b}$ with 
$\sgn({\bf W} {\bf z} +{\bf b}) = \sgn(f({\bf z}) -\tfrac12)$, $f({\bf z})= \bin_s(l({\bf z}))$, for all ${\bf z}\in\{0,1\}^N$. 
That is precisely the statement of Lemma~\ref{lemma:geometricclassification}. 
\end{proof}

We used the following lemma in the proof of Lemma~\ref{lemma:deterministicgeometric}.  For $l=0,1,\ldots, 2^s-1$, the $l$-th orthant of $\R^s$ is the set of all vectors ${\bf r}\in\R^s$ with strictly positive or negative entries and
$\dec({\bf r}_+) = l$, where ${\bf r}_+$ indicates the positive entries of ${\bf r}$. 

\begin{lemma}
	\label{lemma:geometricclassification}
	Let $N=2^s-1$. 
	There is an affine map $\{0,1\}^N\to \R^s$; ${\bf z}\mapsto {\bf W} {\bf z} + {\bf b}$, sending $\{{\bf z}\in\{0,1\}^N\colon l({\bf z})=l \}$ to the $l$-th orthant of $\R^s$, for all $l\in\{0,1,\ldots,N \}$. 
\end{lemma}
\begin{proof}
	Consider the affine map 
	${\bf z}\mapsto {\bf W}{\bf z}+{\bf b}$, 
	where ${\bf b}=-(1,\ldots,1)^\top$ and the $l$-th column of ${\bf W}$ is $2^{l+1}(\bin_s(l)-\frac{1}{2})$ for all $l\in\{1,\ldots, N\}$. 
	For this choice, $\sgn({\bf W} {\bf z} + {\bf b}) = \sgn( \bin_s(l({\bf z})) -\frac{1}{2} )$ lies in the $l$-th orthant of $\R^s$. 
\end{proof}

\begin{figure}
	\centering
	\includegraphics{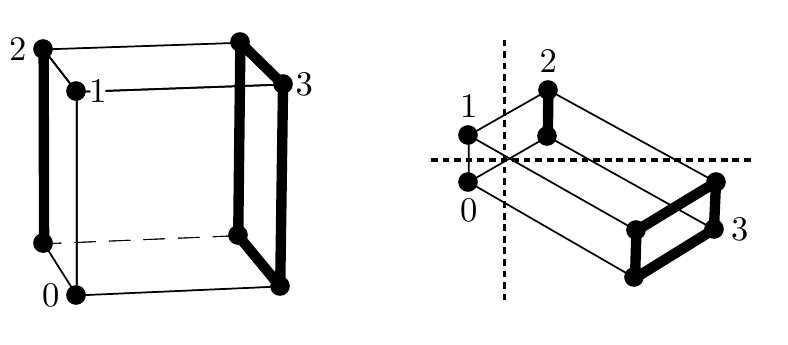}
	\caption{ %
		Illustration of Lemma~\ref{lemma:geometricclassification} for $s=2$, $N=2^s-1 = 3$. 
		There is an arrangement of $s$ hyperplanes which divides the vertices of a $(2^s-1)$-dimensional cube as $0|1|2,3|4,5,6,7|\cdots|2^{2^s-2},\ldots,2^{2^s-1}-1$.}
	\label{figure:geometricclassification}
\end{figure}

Lemma~\ref{lemma:geometricclassification} is illustrated in Figure~\ref{figure:geometricclassification} for $s=2$ and $N=2^s-1=3$. 
As another example, for $s=3$ the affine map can be defined as 
${\bf z}\mapsto {\bf W} {\bf z} + {\bf b}$, where 
\begin{equation*}
{\bf b} =
\begin{pmatrix}
-1 \\ -1 \\ -1
\end{pmatrix},
\quad
\text{and}
\quad
{\bf W} =
\left(
\begin{array}{rrrrrrr}
2 & -4 & 8   & -16 & 32  & -64 & 128\\
-2 &   4 & 8   & -16 & -32& 64   & 128\\
-2 & -4 & - 8&   16 & 32  & 64   & 128
\end{array}
\right).
\end{equation*}

\begin{proposition}
	\label{proposition:univdetmap}
    Let $N=2^s -1$ and let $Q$ be defined as in Lemma~\ref{lemma:deterministicgeometric}. 
	Then $Q\circ \Ecal_{N} = \Delta_{s}^+$, the interior of $\Delta_{s}$ consisting of all strictly positive distributions. 
\end{proposition}
\begin{proof}
	Consider a strictly positive product distribution $p\in \Ecal_N$ with $p({\bf z})=\prod_{i=1}^N p_i^{1-z_i}(1-p_i)^{z_i}$ for all ${\bf z}\in\{0,1\}^N$. 
	Then $p^\top Q\in\Delta_s$ is the vector $q=(q_0,q_1,\ldots, q_N)$ with entries $q_0= p({\bf 0})= \prod_{j=1}^N p_j$ and 
	\begin{eqnarray*}
		q_i 
		&=& \sum_{{\bf z}\colon l( {\bf z})=i} p({\bf z})\\
        &=& \sum_{z_1,\ldots, z_{i-1}} \Big( \prod_{k<i} p_k^{1-z_k}(1-p_k)^{z_k}\Big) (1-p_i) \Big( \prod_{j>i}p_j\Big)\\
		&=& (1-p_i) \prod_{j>i}p_{j}, 
	\end{eqnarray*}
	for all $i=1,\ldots, N$. 
	Therefore, 
	\begin{equation}
	\frac{q_i}{q_0} = \frac{1-p_i}{p_i} \frac{1}{\prod_{j=1}^{i-1}p_j} \quad\forall i=1,\ldots,N. 
	\label{eq:quotients}
	\end{equation}
	Since $\frac{1-p_i}{p_i}$ can be made arbitrary in $(0,\infty)$ by choosing an appropriate $p_i$, 
	independently of $p_j$, for $j<i$,  
	the quotient $\frac{q_i}{q_0}$ can be made arbitrary in $(0,\infty)$ for all $i\in\{1,\ldots, N\}$. This implies that $q$ can be made arbitrary in $\Delta_s^+$. 
	In fact, each $p\in \mathcal{E}_N\cap\Delta_N^+$ is mapped uniquely to one $q\in\Delta_s^+$. 
\end{proof}

\begin{proof}[Proof of Theorem~\ref{theorem:first}]
The statement follows from Proposition~\ref{proposition:firstla} and Proposition~\ref{proposition:univdetmap}. 
\end{proof}

\subsection{Trainable Weights in the Second Layer}
\label{section:proof2}

In order to prove Theorem~\ref{theorem:second} we use the same construction of the first layer as in the previous section, except that we use one block of hidden units for each input vector. The reason for not using a single block for a pair of inputs is that now the second layer will contribute to the modeling of the output distribution in a way that depends on the specific input of the block. 
For the second layer we will use the following refinement of Lemma~\ref{lemma:deterministicgeometric}. 

\begin{lemma}
	\label{lemma:deterministicgeometric2}
	Let $s\geq 2$ and $N=2^{s-1}-1$. 
	The set $\F_{N,s}$ can approximate the following kernels arbitrarily well: 
	\begin{equation*}
	Q(\cdot | {\bf z}) = \lambda_{\bf z} \delta_{\bin_s 2 l({\bf z}) }(\cdot) + (1-\lambda_{\bf z})\delta_{\bin_s 2 l({\bf z}) + 1}(\cdot), \quad \forall {\bf z}\in\{0,1\}^N, 
	\end{equation*}
	where $\lambda_{\bf z}$ are certain (not mutually independent) weights in $[0,1]$. 
	Given any $r_l\in\R_+$, $l\in\{0,1, \ldots, N \}$, 
	it is possible to choose the $\lambda_{\bf z}$'s such that 
	\begin{equation*}
	\frac{\sum_{{\bf z}\colon l({\bf z}) = l} \lambda_{\bf z}}
	{\sum_{{\bf z}\colon l({\bf z}) = l} ( 1 - \lambda_{\bf z})} = r_l,\quad\forall l\in\{0,1,\ldots, N  \}. 
	\end{equation*}  
\end{lemma}

In words, the ${\bf z}$-th row of $Q$ is a convex combination of the indicators of $2l({\bf z})$ and $2l({\bf z})+1$, 
and, furthermore, the total weight assigned to $2l$ relative to $2l+1$ can be made arbitrary for each $l$.  For example, for $s=3$ we have $N=3$ and 
\begin{equation*}
Q = 
\begin{blockarray}{ccccccccc}
000 & \!\!001 & \!\!010 & \!\!011 & \!\!000 & \!\!001 & \!\!010 & \!\!011 & \\
\begin{block}{(cccccccc) c}
\lambda_{\tiny000} & \!\!\!\!(1-\lambda_{\tiny000})&&&&&& & 000\\
&&\!\!\!\!\lambda_{\tiny001}& \!\!\!\!(1-\lambda_{\tiny001}) &&&& & 001\\
&&&&\!\!\!\!\lambda_{\tiny010}& \!\!\!\!(1-\lambda_{\tiny010}) && & 010\\
&&&&\!\!\!\!\lambda_{\tiny011}& \!\!\!\!(1-\lambda_{\tiny011}) && & 011\\
&&&&&&\!\!\!\!\lambda_{\tiny100}& \!\!\!\!(1-\lambda_{\tiny100})  & 100\\
&&&&&&\!\!\!\!\lambda_{\tiny101}& \!\!\!\!(1-\lambda_{\tiny101})  & 101\\
&&&&&&\!\!\!\!\lambda_{\tiny110}& \!\!\!\!(1-\lambda_{\tiny110})  & 110\\
&&&&&&\!\!\!\!\lambda_{\tiny111}& \!\!\!\!(1-\lambda_{\tiny111})  & 111\\
\end{block}
\end{blockarray}
. 
\end{equation*}
The sum of all weights in any given even column can be made arbitrary, relative to the sum of all weights in the column right next to it, for all $N+1$ such pairs of columns simultaneously. 

\begin{proof}[Proof of Lemma~\ref{lemma:deterministicgeometric2}]
    Consider the sets $Z_l=\{{\bf z}\in\{0,1\}^N\colon l({\bf z})=l \}$, for $l=0,1,\ldots, N$. 
	Let ${\bf W}'\in\R^{(s-1)\times N}$ and ${\bf b}'\in\R^{s-1}$ be the input weights and biases defined in Lemma~\ref{lemma:deterministicgeometric}. 
	We define ${\bf W}$ and ${\bf b}$ by appending a row $(\mu_1,\ldots,\mu_N )$ on top of ${\bf W}'$ and an entry $\mu_0$ on top of ${\bf b}'$. 
	
	If $\mu_j<0$ for all $j=0,1,\ldots, N$, then ${\bf z}\mapsto {\bf W} {\bf z} + {\bf b}$ maps $Z_l$ to the $2l$-th orthant of $\R^s$, for each $l=0,1,\ldots, N$. 
	
	Consider now some arbitrary fixed choice of $\mu_j$, for $j< l$. 
	Choosing $\mu_l<0$ with $|\mu_l|>\sum_{j<l}|\mu_l|$, $Z_l$ is mapped to the $2l$-th orthant. 
	If $\mu_l \to -\infty$, then $\lambda_{\bf z} \to 1$ for all ${\bf z}$ with $l({\bf z})=l$. 
	As we increase $\mu_l$ to a sufficiently large positive value, the elements of $Z_l$ gradually are mapped to the $(2l + 1)$-th orthant. 
	If $\mu_l\to\infty$, then $(1 - \lambda_{\bf z})\to 1$ for all ${\bf z}$ with $l({\bf z})=l$. 
	By continuity, there is a choice of $\mu_l$ such that $\frac{\sum_{{\bf z}\colon l({\bf z}) =l} \lambda_{\bf z}}{ \sum_{{\bf z}\colon l({\bf z}) = l} (1-\lambda_{\bf z})} = r_l$. 
	
	Note that the images of $Z_j$, for $j<l$, are independent of the $i$-th columns of ${\bf W}$ for all $i=l,\ldots, N$. 
	Hence changing $\mu_l$ does not have any influence on the images of $Z_l$ nor on $\lambda_{\bf z}$ for ${\bf z}\colon l({\bf z})<l$. 
	Tuning $\mu_i$ sequentially, starting with $i=0$, we obtain a kernel that approximates any $Q$ of the claimed form arbitrarily well. 
\end{proof}

Let $\mathcal{Q}_s^N$ be the collection of kernels described in Lemma~\ref{lemma:deterministicgeometric2}. 
\begin{proposition}
	\label{proposition:sendod}
	Let $s\geq 2$ and $N=2^{s-1} -1$. 
	Then $\mathcal{Q}_s^N \circ \Ecal_{N} = \Delta_{s}^+$. 
\end{proposition}

\begin{proof}
	Consider a strictly positive product distribution $p\in \Ecal_N$ with $p({\bf z}) = \prod_{i=1}^N p_i^{1-z_i}(1-p_i)^{z_i}$ for all ${\bf z}\in\{0,1\}^N$. 
	Then $p^\top Q\in\Delta_s$ is a vector $(q_0,q_1,\ldots, q_{2N+1})$ whose entries satisfy 
	$q_0 + q_1 = p({\bf 0})= \prod_{j=1}^N p_j$ and 
	\begin{eqnarray*}
		q_{2i}+q_{2i+1} = (1-p_i) \prod_{j>i}p_{j}, 
	\end{eqnarray*}
	for all $i = 1, \ldots, N$. 
	As in the proof of Proposition~\ref{proposition:univdetmap}, this implies that the vector $(q_0+q_1, q_2+q_3,\ldots, q_{2N}+q_{2N+1})$ can be made arbitrary in $\Delta_{s-1}^+$. 
	This is irrespective of the coefficients $\lambda_0,\ldots, \lambda_N$. 
	Now all we need to show is that we can make $q_{2i}$ arbitrary relative to $q_{2i+1}$ for all $i=0,\ldots, N$. 
	
	We have 
	\begin{eqnarray*}
		q_{2i} 
		&=&\sum_{{\bf z}\colon l({\bf z})=i} \lambda_{\bf z} p({\bf z})\\
		&=& \left(\sum_{{\bf z}\colon l({\bf z})=i}\lambda_{{\bf z}} \Big( \prod_{k<i} p_k^{1-z_k}(1-p_k)^{z_k}\Big)\right)
		(1-p_i) \Big( \prod_{j>i}p_j\Big) , 
	\end{eqnarray*}
	and
	\begin{eqnarray*}
		q_{2i+1} 
		&=&\sum_{{\bf z}\colon l({\bf z})=i} (1 - \lambda_{\bf z}) p({\bf z})\\
		&=& \left(\sum_{{\bf z}\colon l({\bf z})=i}(1 -\lambda_{{\bf z}}) \Big( \prod_{k<i} p_k^{1-z_k}(1-p_k)^{z_k}\Big)\right) 
		(1-p_i) \Big( \prod_{j>i}p_j\Big). 
	\end{eqnarray*}
	Therefore, 
	\begin{eqnarray*}
		\frac{q_{2i}}{q_{2i+1}} 
		&=& 
		\frac{ \sum_{{\bf z}\colon l({\bf z})=i}\lambda_{{\bf z}} \Big( \prod_{k<i} p_k^{1-z_k}(1-p_k)^{z_k}\Big) }
		{\sum_{{\bf z}\colon l({\bf z})=i}(1 -\lambda_{{\bf z}}) \Big( \prod_{k<i} p_k^{1-z_k}(1-p_k)^{z_k}\Big)}. 
	\end{eqnarray*}
	By Lemma~\ref{lemma:deterministicgeometric2} it is possible to choose all $\lambda_{\bf z}$ arbitrarily close to zero for all ${\bf z}$ with $l({\bf z})=i$ and have them transition continuously to values arbitrarily close to one (independently of the values of $\lambda_{\bf z}$, ${\bf z}\colon l({\bf z})\neq i$). 
	Since all $p_k$ are strictly positive, this implies that the quotient $\frac{q_{2i}}{q_{2i+1}}$ takes all values in $(0,\infty)$ as the $\lambda_{\bf z}$, ${\bf z}\colon l({\bf z})=i$ transition from zero to one. 
\end{proof}

\begin{proof}[Proof of Theorem~\ref{theorem:second}]
	This follows from a direct modification of Proposition~\ref{proposition:firstla} to have a hidden block per input vector, and Proposition~\ref{proposition:sendod}. 
\end{proof}

\subsection{Discussion of the Proofs for Shallow Networks}
\label{section:conclusion}

We proved upper bounds on the minimal size of shallow sigmoid stochastic feedforward networks that can approximate any stochastic function with a given number of binary inputs and outputs arbitrarily well. 
By our analysis, if all parameters of the network are free, $2^d(2^{s-1}-1)$ hidden units suffice, and, if only the parameters of the first layer are free, $2^{d-1}(2^s-1)$ hidden units suffice. 

It is interesting to compare these results with what is known about universal approximation of Markov kernels by undirected stochastic networks called conditional restricted Boltzmann machines. For those networks, \citet{montufar2015geometry} showed that $2^{d-1}(2^s-1)$ hidden units suffice, whereby, if the number $d$ of input units is large enough, $\frac14 2^{d}(2^s-1 + 1/30)$ suffice. 
A more recent work \citep{montufar2017hierarchical} showed that an RBM is a universal approximator as soon as $m\geq \frac{2(\log(v) + 1)}{(v + 1)}2^{v} -1$, where $v=d+s$, which has a smaller asymptotic behavior. 
In the case of no input units, our bound $2^{s-1}-1$ equals the bound for RBMs from~\cite{montufar2011expressive}, but is larger than the bound from \cite{montufar2017hierarchical}. 
It has been observed that undirected networks can represent many kernels that can be represented by feedforward networks, especially when these are not too stochastic~\citep{montufar2015geometry,montufar2014deep}. 
Verifying the tightness of the bounds remains an open problem, as well as the detailed comparison of directed and undirected architectures.

\section{Results for Deep Networks} 
\label{sec:deep}

We now consider networks with multiple layers of hidden units, i.e. $L > 1$. 
Since the dimension of $\D_{d,s}$ is $2^d(2^s - 1)$, a lower bound on the number of trainable parameters a model needs for universal approximation is $2^d(2^s - 1)$. Details on this are provided in Section~\ref{sec:lowerbounds}. 
The following theorem provides sufficient conditions for a spectrum of deep architectures to be universal approximators. 

\begin{theorem}
\label{theorem:deep2}
Let $d,s\in\N$, and assume that $s = 2^{b-1} + b$ for some $b \in \N$. 
Then a deep sigmoid stochastic feedforward network with $d$ binary inputs and $s$ binary outputs is a universal approximator of Markov kernels in $\D_{d,s}$ if it contains 
$2^{d-j} ( 2^{s-b} + 2^b - 1 )$
hidden layers each consisting of $2^j(s+d-j)$ units, for any $j \in \{0,1,2,\dots,d\}$. 
\end{theorem}

We note that the indicated upper bound on the width and depth holds for universal approximation of $\Delta_{d,s'}$ for any $s'\leq s$. Moreover, if $j=d$, we can save one layer. We will make use of this in our numerical example in Section~\ref{sec:examples}. 
One can also use a simplified construction with $2^{d+s-j}$ hidden layers. 

The theorem indicates that there is a spectrum of networks capable of universal approximation, where if the network is made narrower, it must become proportionally deeper. 
The network topology with $j = d$ has depth exponential in $s$ and width exponential in $d$, whereas  $j = 0$ has depth exponential in $d$ and $s$, but only width $d+s$. See Figure~\ref{fig:example11} for a sketch of how $j$ affects the different network properties and shape. 
There may exist a spectrum of networks bridging the gap between the shallow universal approximators from Section~\ref{sec:shallow} which have width exponential in $d$ and $s$ and only one hidden layer, and the $j = d$ case, although no formal proof has been established. 

When considered as fully connected, the networks described in Theorem~\ref{theorem:deep2} vary greatly in their number of parameters. 
However, we will see that each of them can be implemented with the same minimal number of trainable parameters, equal to the dimension of $\Delta_{d,s}$. 
For the most narrow case $j=0$, each input vector is carried to a specific block of hidden layers which create the corresponding output distribution, which then this is passed downwards until the output layer. 
When $j > 0$ it will be shown that the first layer can divide the input space into $2^j$ sets, each to be handled by its own parallel section of the network. 
Each of the $2^j$ sections can be thought of as running side-by-side and non-interacting, meanwhile creating the corresponding output distributions of the $2^{d-j}$ different inputs. 
In the widest case $j=d$, each input is processed by its own parallel section of the network, and $\frac{2^s}{2(s-b)} + 2(s-b) -1$ hidden layers are sufficient for creating the corresponding output distribution, which resembles the the bound for deep belief networks \citep{montufar2011refinements}. 

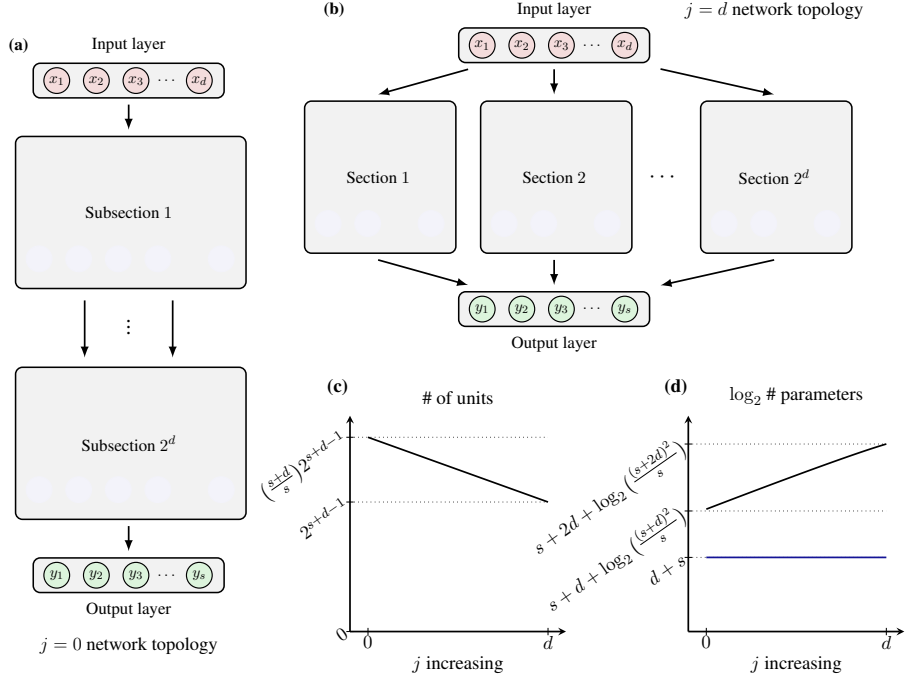
\begin{figure}[ht]
\centering
\begin{tikzpicture}
\node[
] (picutre1) at (-3.1,6.5)
    {
    \scalebox{.65}{
        \tikzset{node distance=2cm, auto}
        \begin{tikzpicture}[scale=0.9, every node/.style={transform shape}]
        \node (L-3) at (0,4) {\labeledlayer{x}{3}{{d}}{ }{ic} };
        \node (L-2b) at (0,1) {\unitlayerr{h}{4}{}{}{bl}};
        \node (L-2c2) at (0,-1.5) {\Large$\vdots$}; 
        \node (L-2d) at (0,-4.25) {\unitlayerr{h}{4}{}{}{bl}}; 
        \node (L-0) at (0,-7.25) {\labeledlayer{y}{3}{{s}}{ }{oc} };

        \draw[-latexnew, arrowhead=.225cm, line width=1pt] (1,-1) to (1,-2.3);
        \draw[-latexnew, arrowhead=.225cm, line width=1pt] (-1,-1) to (-1,-2.3);

        \draw[-latexnew, arrowhead=.225cm, line width=1pt] (L-3) to (L-2b); 
        \draw[-latexnew, arrowhead=.225cm, line width=1pt] (L-2d) to (L-0);

        \node (o) [node distance=.8cm, below of = L-0] {Output layer};
        \node (i) [node distance=.8cm, above of = L-3] {Input layer};
        \node (tit) [node distance=.8cm, below of = o] {\large $j = 0$ network topology};

        \node (last) [node distance=.00cm, below of = L-2d] {Subsection $2^d$};
        \node (one) [node distance=.00cm, below of = L-2b] {Subsection $1$};
        \node (ah) [node distance=2.5cm, left of = i] {\textbf{(a)}};
        \end{tikzpicture}
        }
    };

\node[
] (pitcutre2) at (2.9,8.75)
    {
    \scalebox{.65}{
        \tikzset{node distance=2cm, auto}
        \begin{tikzpicture}[scale=0.9, every node/.style={transform shape}]
        \node (L-3) at (0,4) {\labeledlayer{x}{3}{{d}}{ }{ic} };
        \node (L-2a) at (-4,1) {\unitlayerr{h}{2}{}{}{bl}}; 
        \node (L-2b) at (0,1) {\unitlayerr{h}{2}{}{}{bl}}; 
        \node (L-2c) at (5,1) {\unitlayerr{h}{2}{}{}{bl}};   
        \node (L-2d) at (2.5,1) {\Large$\cdots$}; 
        \node (L-0) at (0,-2) {\labeledlayer{y}{3}{{s}}{ }{oc} };
        \draw[-latexnew, arrowhead=.225cm, line width=1pt] (L-3)  to (L-2a.north);
        \draw[-latexnew, arrowhead=.225cm, line width=1pt] (L-3)  to (L-2b.north); 
        \draw[-latexnew, arrowhead=.225cm, line width=1pt] (L-3)  to (L-2c.north); 
        \draw[-latexnew, arrowhead=.225cm, line width=1pt] (L-2a.south) to (L-0); 
        \draw[-latexnew, arrowhead=.225cm, line width=1pt] (L-2b.south) to (L-0); 
        \draw[-latexnew, arrowhead=.225cm, line width=1pt] (L-2c.south) to (L-0); 
        \node (o) [node distance=.8cm, below of = L-0] {Output layer};
        \node (i) [node distance=.8cm, above of = L-3] {Input layer};
        \node (tit) [node distance=5cm, right of = i] {\large $j = d$ network topology}; 
        \node (one) [node distance=.00cm, below of = L-2a] {Section $1$};
        \node (two) [node distance=.00cm, below of = L-2b] {Section $2$};
        \node (last) [node distance=.00cm, below of = L-2c] {Section $2^d$};
        \node (bee) [node distance=5cm, left of = i] {\textbf{(b)}};
        \end{tikzpicture}}
    };

\node[inner sep=0pt] (pitcutre3) at (0.7,4.1)
    {
        \scalebox{0.65}{
        \begin{tikzpicture}
        \node (see) at (-.25,5) {\textbf{(c)}};
        \begin{axis}[width=6cm,height=6cm,line width =1pt,
            title={\# of units},
            xlabel={$j$ increasing},
            xmin=-1, xmax=11,
            ymin=0, ymax=100,
            xtick={0},
            extra x ticks= {10},
            extra x tick labels={$d$},
            ytick={0},
            extra y ticks= {60,90},
            extra y tick labels={${2^{s+d-1}}$,$\big(\frac{s+d}{s}\big)2^{s+d-1}$},
            yticklabel style={rotate=35,anchor=east}, 
            every tick/.style={black,semithick},            
            axis x line=bottom,
            axis y line=left,
        ]
         
        \addplot [domain=-1:10, samples=100 ,color=black, thin, dotted]{60};    
        \addplot [domain=-1:10, samples=100 ,color=black, thin, dotted]{90};    
        
        \addplot[color=black,]
            coordinates {
            (0,90)(10,60)
            }; 
        \end{axis}
        \end{tikzpicture}
        }
    };

\node[inner sep=0pt] (pitcutre4) at (4.75,4.1)
    {
    \scalebox{0.65}{
    \begin{tikzpicture}
    \node (dee) at (-.25,5) {\textbf{(d)}};
    \begin{axis}[width=6cm,height=6cm,line width =1pt,
        title={$\log_2$ \# parameters},
        xlabel={$j$ increasing},
        xmin=-1, xmax=11,
        ymin=1, ymax=100,
        xtick={0},
        extra x ticks= {10},
        extra x tick labels={$d$},
        ytick={0},
        extra y ticks= {35, 56.3,87},
        extra y tick labels={$d+s$, $s+d + \log_2\big(\frac{(s+d)^2}{s}\big)$,$s+2d + \log_2\big(\frac{(s+2d)^2}{s}\big)$},
        yticklabel style={rotate=35,anchor=east},
        every tick/.style={black,semithick},
        axis x line=bottom,
        axis y line=left,
    ]
     
    \addplot [domain=-1:10, samples=100 ,color=black, thin, dotted]{35};    
    \addplot [domain=-1:10, samples=100 ,color=black, thin, dotted]{56.3};
    \addplot [domain=-1:10, samples=100 ,color=black, thin, dotted]{87};
    
    \addplot [domain=0:10, samples=100 ,color=black]{log2(9^(x) * (140 - x^2)) + 50};
    \addplot [domain=0:10, samples=100 ,color=bl]{35};
    
    \end{axis}
    \end{tikzpicture}
    }
    };
\end{tikzpicture}
\caption{(a) The deepest narrowest network architecture $j = 0$.  Here, there is a single section consisting of $2^d$ subsections stacked on top of each other. 
(b) The widest deep architecture $j = d$.  Here, there are $2^d$ sections placed in parallel, each consisting of a single subsection. 
(c) A sketch of how the number of units scales as a function of $j$, for fixed $s,d$. (d) A log-scale sketch of how the total number of network parameters scales with $j$. 
Blue shows the rounded log number of trainable parameters in our construction, which is independent of $j$. 
}
\label{fig:example11}
\end{figure}

In some special instances, small reductions in size are possible. For instance, in the widest network topology considered, where $j = d$, universal approximation can be done with one less layer.

\subsection*{Parameter Count}

If we are to consider the networks as being fully connected and having all weight and bias parameters trainable, the number of parameters grows exponentially with $j$ as 
\begin{align*}
    |\th_{\text{full}}| 
    &= \sO\bigg(\frac{2^{s+d+j}(s+d-j)^2}{2(s-\log_2(s))}\bigg),
\end{align*}
suggesting that the deepest topology where $j = 0$ is the most efficient universal approximator considered. The number of units for these networks linearly decrease with $j$ as
\begin{equation*}
\text{\# of units} = \sO\bigg(\frac{2^{s+d}(s+d-j)}{2(s-\log_2(s))}\bigg),
\end{equation*}
meaning that the widest topology, where $j=d$, uses the least number of units. However, if one counts only the parameters that can not be fixed prior to training according to the construction that we provide below in Section~\ref{sec:deep_proofs}, one obtains 
\begin{equation*}
    |\th_{\text{trainable}}| = 2^d(2^s-1). 
\end{equation*}
Each of the network topologies has the hidden units organized into $2^j$ sections and a total of $2^d$ subsections, each with $2^s-1$ trainable parameters. 
The first hidden layer and the output layer have fixed parameters. 
Each one of the trainable parameters controls exactly one entry in the Markov kernel and number of parameters is also necessary for universal approximation of $\D_{d,s}$, as we will show further below in Section~\ref{sec:lowerbounds}.

\subsection*{Approximation with Finite Weights and Biases}

The quality of approximation provided in our construction depends on the magnitude of the network parameters. 
If this is allowed to increase unboundedly, a universal approximator will be able to approximate any Markov kernel with arbitrary accuracy. 
If the parameters are only allowed to have a certain maximal magnitude, the approximation is within an error bound described in the following theorem. 

\begin{theorem} \label{thrm:errr}
Let $\e \in (0,1/2^s)$ and consider a target kernel $p^\ast$ in the non-empty set $\D_{d,s}^\e := \{P \in \D_{d,s} \colon \e \leq P_{ij} \leq 1-\e \text{ for all } i,j\}$. 
There is a choice of the network parameters, bounded in absolute value by $\a = 2m\s^{-1}(1-\e)$, where $m =\max\{j,s+(d-j)\}\leq d+s$, 
such that the conditional probabilities $p({\bf y}|{\bf x})$ generated by the network are uniformly close to the target values according to 
\begin{equation*}
|p({\bf y} | {\bf x}) - p^*({\bf y} | {\bf x})| \leq 1 - (1-\e)^N  
+ \e,\quad \forall {\bf x}\in\{0,1\}^d, {\bf y}\in\{0,1\}^s,  
\end{equation*}
where $N$ is the total number of units in the network excluding input units. 
If one considers an arbitrary target kernel $p^\ast$, the error bound increases by $\e$. 
\end{theorem}

The proof of Theorem~\ref{thrm:errr} is presented in Section~\ref{section:errr} after the proof of Theorem~\ref{theorem:deep2}. It depends on explicit error bounds for the probability sharing steps discussed next.

\section{Proofs for Deep Networks} 
\label{sec:deep_proofs}

The proof naturally splits into three steps. 
The first step shows that the first layer is capable of dividing the input space into $2^j$ disjoint sets of $2^{d-j}$ input vectors, 
sending each set to a different parallel running section of the network. 
The $2^{d-j}$ vectors in the $\t$-th set will activate the $\t$-th section, while the other sections take state zero with probability one. 
Second, it is shown that each of the $2^j$ sections is capable of approximating the conditional distributions for the corresponding $2^{d-j}$ inputs. The last step explicitly determines the parameters to copy the relevant units from the last hidden layer to the output layer.

\subsection{Notation}
\label{section:Notation}

The integer $j$ dictates the network's topology. 
This index can be any number in $\{0,1,\dots,d\}$. 
For an input vector ${\bf x}$, we denote the $r$ through $r'$ bits by ${\bf x}_{[r,r']}$. 
The target conditional probability distribution given the input ${\bf x}$ is denoted by $p^*(\cdot | {\bf x})$. 
The joint state of all units in the $l$-th hidden layer will be denoted by ${\bf h}^l$. 
The state of the $r$-th unit of the $l$-th layer is denoted by ${\bf h}_r^l$. 
If a range is provided as superscript (subscript), it is referring to a range of layers (units), i.e. ${\bf h}^{[l,l+2]}$ refers to the hidden layers ${\bf h}^l, {\bf h}^{l+1}, {\bf h}^{l+2}$. 
The integer $\t = 1,2,\dots,2^{j}$ will be an index specifying a \textit{block} of the units in a layer. 
Each block consists of $s+d-j$ consecutive units. 
The $\t$-th block of units in a given layer are those indexed from $(\t-1)(s+d-j)+1$ to $\t(s+d-j)$. 
The state of $\t$-th block of units of hidden layer $l$ is ${\bf h}_{(\t)}^l := {\bf h}^l_{[(\t-1)(s+d-j)+1,\,\t(s+d-j)]}$. 
A block \textit{activates} means that it can take a state other than the zero vector with non-zero probability. 
If the block is \textit{inactive}, it will take the zero state with arbitrarily high probability. 
Due to their different function in the following, 
it is useful to denote 
${\bf a}_{(\t)}^l$
and 
${\bf b}^l_{(\t)}$ 
the first $s$ and the last $d-j$ units of the $\t$-th block at the $l$-th layer. 
The \nth{1} unit of ${\bf a}_{(\t)}^l$ is denoted as ${\bf a}_{(\t),1}^l$, and plays an important role in the first layer of the network.

For the second part of the proof, we will first show that $L = 2^{s+d-j}$ hidden layers suffice, and then refine this bound.  Thus the focus will be on the entire $\t$-th \textit{section} of the hidden units. 
The $\t$-th section are the units in the $\t$-th block over all of the hidden layers, denoted without a superscript ${\bf h}_{(\t)} := {\bf h}^{[1,2^{s+d-j}]}_{(\t)}$. 
Each section will be broken into subsections, indexed by $q = 1,2,\dots,  2^{d-j}$. 
The $q$-th subsection of the $\t$-th section is ${\bf h}_{(\t)}^{(q)} := {\bf h}_{(\t)}^{[(q-1)2^s+1,q2^s]}$. 
Last, each subsection of the network will have a Gray code associated with it, i.e., a sequence of binary vectors with subsequent vectors differing in only one bit.  In actuality, since each subsection will be capable of performing not just one but rather multiple tasks per layer, each subsection will have a set of partial Gray codes associated with it, to be defined in the following.
See Figure~\ref{fig:zero} for an illustration of the notation described here. 

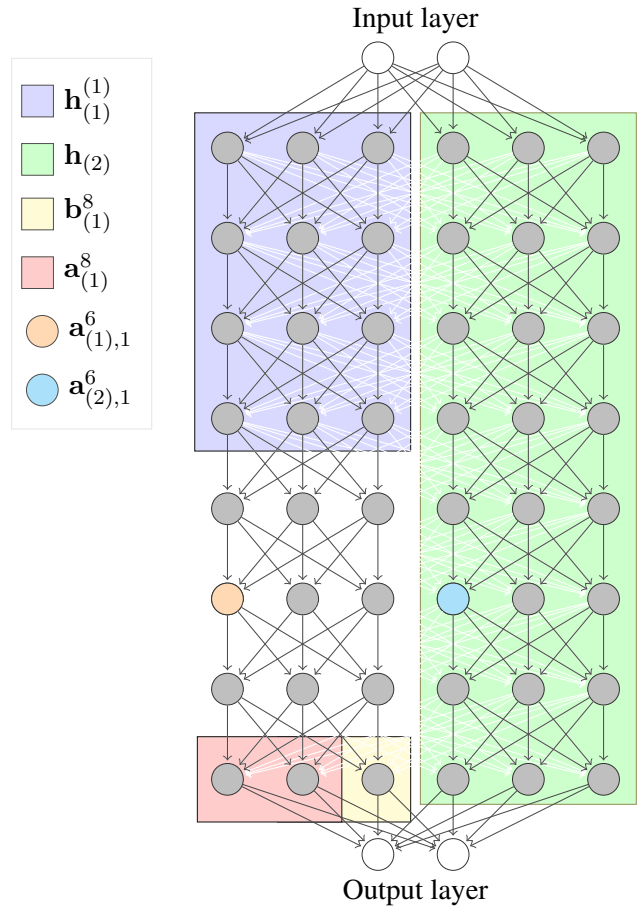
\begin{figure}
\centering
\begin{tikzpicture}[shorten >=1pt,->,draw=black!10, node distance=\layersep]
   
    \tikzstyle{every pin edge}=[<-,shorten <=1pt]
    \tikzstyle{neuron}=[circle,fill=black!25,minimum size=12pt,inner sep=0pt]
    \tikzstyle{input neuron}=[neuron, draw=black!80, line width=0.1mm, fill = black!0!white];
    \tikzstyle{output neuron}=[neuron, draw=black!80, line width=0.1mm, fill = black!0!white];
    \tikzstyle{hidden neuron}=[neuron, draw=black!80, line width=0.1mm, fill = black!25!white];
    \tikzstyle{annot} = [text width=4em, text centered]

    \foreach \name / \y in {1,...,2}
        \node[input neuron] (I-\name) at (-\y-2,10) {};
    \foreach \name / \y in {1,...,6}
        \path node[hidden neuron] (H1-\name) at (-\y, 8.8 cm) {};
    \foreach \name / \y in {1,...,6}
        \path node[hidden neuron] (H2-\name) at (-\y, 7.6cm) {};
    \foreach \name / \y in {1,...,6}
        \path node[hidden neuron] (H3-\name) at (-\y, 6.4cm) {};
    \foreach \name / \y in {1,...,6}
        \path node[hidden neuron] (H4-\name) at (-\y, 5.2cm) {};
    \foreach \name / \y in {1,...,6}
        \path node[hidden neuron] (H5-\name) at (-\y, 4 cm) {};
    \foreach \name / \y in {1,...,6}
        \path node[hidden neuron] (H6-\name) at (-\y, 2.8cm) {};
    \foreach \name / \y in {1,...,6}
        \path node[hidden neuron] (H7-\name) at (-\y, 1.6cm) {};
    \foreach \name / \y in {1,...,6}
        \path node[hidden neuron] (H8-\name) at (-\y, 0.4cm) {};
    \foreach \name / \y in {1,...,2}
        \node[output neuron] (O-\name) at (-\y-2, -0.6 cm) {};

    \path node[hidden neuron, orange!30, draw=black!80, line width=0.1mm] (a_1) at (-6,2.8cm) {};
    \path node[hidden neuron, cyan!30, draw=black!80, line width=0.1mm] (a_2) at (-3,2.8cm) {};
    
    \foreach \source in {1,...,2}
        \foreach \dest in {1,...,6}
            \path (I-\source) edge[black!70] (H1-\dest);
   
    \foreach \layer [count=\xi from 2] in {1,...,7}
    	\foreach \source in {1,...,6}
        	\foreach \dest in {1,...,6}
            	\path (H\layer-\source) edge[white] (H\xi-\dest);

    \foreach \layer [count=\xi from 2] in {1,...,7}
    	\foreach \source in {1,2,3}
        	\foreach \dest in {1,2,3}
            	\path (H\layer-\source) edge[black!70] (H\xi-\dest);

    \foreach \layer [count=\xi from 2] in {1,...,7}
    	\foreach \source in {4,5,6}
        	\foreach \dest in {4,5,6}
            	\path (H\layer-\source) edge[black!70] (H\xi-\dest);

    \foreach \source in {1,...,6}
        \foreach \dest in {1,...,2}
            \path (H8-\source) edge[black!70] (O-\dest);

    \node at (-5,8.7 cm) [
    add rectangle=4.5cm
    ] {
    \begin{varwidth}{2.6cm}
    $$\hspace{3.2cm}$$
    \end{varwidth}
    };

    \node at (-2,8.7 cm) [
    add Rectangle=9.2cm
    ] {
    \begin{varwidth}{2.6cm}
    $$\hspace{3.2cm}$$
    \end{varwidth}
    };

    \node at (-4.45,0.4 cm) [
    add RRectangle=0.8cm
    ] {
    \begin{varwidth}{1.5cm}%
    $$\hspace{1.5cm}$$
    \end{varwidth}%
    };

    \node at (-5.44,0.4 cm) [
    add RRectanglee=0.8cm
    ] {
    \begin{varwidth}{1.655cm}%
    $$\hspace{1.65cm}$$
    \end{varwidth}%
    };

    \tikzstyle{section}=[rectangle,fill=green!20,minimum size=12pt,inner sep=0pt,line width=0.1mm, draw=black!70]
    \tikzstyle{subsection}=[rectangle,fill=blue!15,minimum size=12pt,inner sep=0pt,line width=0.1mm, draw=black!70]
    \tikzstyle{ablock}=[rectangle,fill=red!20,minimum size=12pt,inner sep=0pt,line width=0.1mm, draw=black!70]
    \tikzstyle{bblock}=[rectangle,fill=yellow!20,minimum size=12pt,inner sep=0pt,line width=0.1mm, draw=black!70]
    \tikzstyle{noder}=[circle,fill=orange!30,minimum size=12pt,inner sep=0pt,line width=0.1mm, draw=black!70]
    \tikzstyle{noder2}=[circle,fill=cyan!30,minimum size=12pt,inner sep=0pt,line width=0.1mm, draw=black!70]
    
    \matrix [draw,below left] at (-7,10) {
  	\node [subsection,label=right:${\bf h}_{(1)}^{(1)}$] {}; \\
  	\node [section,label=right:${\bf h}_{(2)}$] {}; \\
  	\node [bblock,label=right:${\bf b}_{(1)}^8$] {}; \\
  	\node [ablock,label=right:${\bf a}_{(1)}^8$] {}; \\
  	\node [noder,label=right:${\bf a}_{(1),1}^6$] {}; \\
  	\node [noder2,label=right:${\bf a}_{(2),1}^6$] {}; \\
	};
    \node[above of=I-2, node distance=.5cm] {\hspace{1cm}Input layer};
    \node[below of=O-2, node distance=.5cm] {\hspace{1cm}Output layer};
\end{tikzpicture}
\caption{The network architecture from Theorem~\ref{theorem:deep2} for $j = 1$, with $d = 2$ inputs and $s = 2$ outputs. 
The figure exemplifies the notations defined in Section~\ref{section:Notation}. 
The light connections are set to zero, separating the hidden units into $2^{d-j}$ parallel running sections. 
Here there are 2 sections indexed by $\t=1,2$, and in each section there are 2 subsections indexed by $q = 1,2$. 
The network will have $2^d$ subsections in total when added across all sections. 
Each output vector is generated with an appropriate conditional probability given the input vector ${\bf x}$, 
by mapping the input through a corresponding sequence of states of ${\bf a}_{(\t)}$ with appropriate probability, 
and information about the input is preserved throughout the hidden layers by ${\bf b}_{(\t)}$.}
\label{fig:zero}
\end{figure}

\subsection{Probability Mass Sharing}

For fixed weights and a given input, a width-$m$ layer of the network will exhibit a marginal distribution $p \in \D_m$. 
A subsequent width-$m$ layer determines a particular mapping of $p$ to another distribution $p' \in \D_m$. 
For certain choices of the parameters, this mapping transforms $p$ in such a way that a fraction of the mass of a given state ${\bf g}\in\{0,1\}^m$ is transferred to some other state $\hat{\bf g}\in\{0,1\}^m$, so that $p'(\hat{\bf g}) = p(\hat {\bf g}) + \lambda p({\bf g})$, $p'({\bf g}) = (1-\lambda) p({\bf g})$, and 
$p'({\bf z}) = p({\bf z})$ for all other states ${\bf z}$. 
This mapping is referred to as probability mass sharing, and was exploited in the works of \cite{sutskever2008deep, le2010deep, montufar2011refinements}. 
One important takeaway from these works is that probability mass sharing in one layer is restrictive and the states ${\bf g}$ and $\hat{\bf g}$ need to stand in a particular relation to each other, e.g., being Hamming neighbors. \citet{78096} give a description of the mappings that are expressible by the individual layers of a Bayesian sigmoid belief network.  In the following we describe ways in which probability mass sharing is possible. 

We define Gray codes and partial Gray codes as they will be useful in discussion of probability mass sharing sequences. 
A \textit{Gray code} ${\cal G}$ for $s$ bits is a sequence of vectors ${\bf g}_i\in \{0,1\}^s$, $i=1,\ldots, 2^s$, such that $\|{\bf g}_i - {\bf g}_{i-1}\|_H = 1$ for all $i \in \{2,\dots,2^s\}$, and $\cup_i \{{\bf g}_i\} = \{0,1\}^s$. 
One may visualize such a code as tracing a path along the edges of the $s$-cube without ever returning to the same vertex and covering all the vertices. 
A \textit{partial Gray code} ${\cal S}_i$ for $s$ bits is a sequence of vectors ${\cal S}_{i,j}\in\{0,1\}^s$, $j=1,\ldots, r\leq 2^s$, such that $\|{\cal S}_{i,j} - {\cal S}_{i,j-1}\|_H=1$ for all $j \in \{1,\ldots,r\}$. 
We will be interested in collections of partial Gray codes which contain the same number of vectors and partition $\{0,1\}^s$. 
The analysis will have subsections of the network activate to the values of different partial Gray codes with appropriate probabilities. 

The first proposition states that when considering two consecutive layers $l-1$ and $l$, there exists a weight vector ${\bf W}^l_i$ and bias $b_{i}^l$ such that $h_{i}^l$ will be a copy of $h_i^{l-1}$ with arbitrarily high probability. 
See Section 3.2 of \cite{sutskever2008deep} or Section 3.3 of \cite{le2010deep} for the equivalent statements. 
\begin{proposition} 
\label{prop:sharing1}
    Fix $\e \in (0,1/2)$ and $\a = \log{(1-\e)} - \log{(\e)}$. 
    Choose the $i$-th row of weights ${\bf W}_i^l$ such that $W_{ii}^l = 2\a$ and $W_{ij}^l = 0$ for $j \neq i$, and choose the $i$-th bias $b_i^l = -\a$. 
    Then $\Pr(h_{i}^l = h_{i}^{l-1} | {\bf h}^{l-1}) = (1-\e)$ for all ${\bf h}^{l-1}$. 
    Letting $\a \to \infty$ allows for the unit to copy the state of a unit in the previous layer with arbitrarily high probability. 
\end{proposition}
\begin{proof}
For the given choice of weights and bias, the total input to the $i$-th unit in layer $l$ will be $(2\a h_i^{l-1} - \a)$, 
meaning
\bq
    \Pr(h_{i}^l = 1 | {\bf h}^{l-1}) = \begin{cases} \s(\a) & \text{ when } h_i^{l-1} = 1 \\ \s(-\a) & \text{ when } h_i^{l-1} = 0\end{cases} . 
\eq 
Since $\a = \s^{-1}(1-\e)$, one has that $\Pr(h_{i}^l = h_{i}^{l-1} | {\bf h}^{l-1}) = (1-\e)$. 
\end{proof}

The next theorem states that the weights ${\bf W}_i^l$ and bias $b_i^l$ of the $i$-th unit in layer $l$ may be chosen such that, if ${\bf h}^{l-1}$ matches any of two pre-specified vectors ${\bf g}$ and $\hat{\bf g}$ in $\{0,1\}^m$, then $h_i^l$ flips $h_i^{l-1}$ with a pre-specified probability. Otherwise, if ${\bf h}^{l-1}$ is not ${\bf g}$ or $\hat{\bf g}$, the bit is copied with arbitrarily high probability. 
This corresponds to Theorem~2 of \citet{le2010deep} adjusted to our notation. 

\begin{theorem}[Theorem~2 of \citealp{le2010deep}]
\label{theorem:sharing1}
Consider two subsequent layers $l-1$ and $l$ of width $m$. 
Consider two vectors ${\bf g}$ and ${\hat {\bf g}}$ in $\{0,1\}^m$ with $\|{\bf g} - {\hat {\bf g}}\|_H = 1$, differing in entry $j$. 
Fix two probabilities $\rho,\hat\rho \in (0,1)$ and a tolerance $\epsilon \in (0,1/2)$. 
For any $i\neq j$, there exist $({\bf W}_i^l, b_i^l)$ with absolute values at most $\alpha = 2m(\sigma^{-1}(1-\epsilon))$ 
such that 
\begin{align*}
\Pr(h_i^l &= 1 \, | \, {\bf h}^{l-1} = {\bf g}) = \rho_\e,\\ 
\Pr(h_i^l &= 1 \, | \, {\bf h}^{l-1} = \hat{\bf g}) = \hat \rho_\e, 
\intertext{and, otherwise,}
\Pr(h_i^l &= h_i^{l-1}\, | \, {\bf h}^{l-1} \neq {\bf g}, \hat {\bf g}) = (1-\epsilon).  
\end{align*}
Here $\rho_\epsilon=\max\{\epsilon, \min\{\rho,1-\epsilon\}\}$. 
If we fix a maximum parameter magnitude $\alpha$ instead of a tolerance, then we can substitute $\epsilon = 1-\sigma(\frac{\alpha}{2m})$. 
\end{theorem}

This theorem allows to have a given vector ${\bf g}\in\{0,1\}^m$ map at the subsequent layer to itself or to a Hamming adjacent vector ${\bf g}'\in\{0,1\}^m$ with a pre-specified probability $\rho$, with 
\begin{align}
\Pr({\bf h}^l = {\bf g} \, | \, {\bf h}^{l-1} = {\bf g}) &= \rho_\e (1-\epsilon)^{m-1},
\label{eq:sharingstep_a}\\
\Pr({\bf h}^l = {\bf g}' \, | \, {\bf h}^{l-1} = {\bf g}) &= (1-\rho_\e) (1-\epsilon)^{m-1},
\intertext{and}
 \Pr({\bf h}^l = {\bf h}^{l-1}\, | \, {\bf h}^{l-1} \neq {\bf g}) &= (1-\epsilon)^m, 
 \label{eq:sharingstep}
 \end{align}
where $\e$ can be made arbitrarily small if the maximum magnitude $\alpha = 2m(\sigma^{-1}(1-\epsilon))$ of weights and biases is allowed to grow to infinity. 
This mapping is referred to as a \textit{probability mass sharing step}, or a sharing step for short. 
This in turn allows to transfer probability mass around the $m$-cube one vertex at a layer until the correct probability mass resides on each binary vector, to the given level of accuracy. The sharing path follows a Gray code with each pair of consecutive vectors having Hamming distance one.

In fact, the theorem allows us to overlay multiple sharing paths, so long as the Gray codes are sufficiently separated. 
A collection of partial codes satisfying this requirement and covering the set of binary strings is described in the following theorem, which is Lemma 4 of \citet{montufar2011refinements}. 

\begin{theorem}[Lemma 4 of \citealp{montufar2011refinements}] 
\label{theorem:sharing2}
    Let $m = \frac{2^b}{2} + b, b\in\N, b\geq 1$. There exist $2^b = 2(m-b)$ sequences $\sS_{i}$, $1 \leq i \leq 2^b$, composed of length-$m$ binary vectors $\sS_{i,k}$, $1\leq k\leq 2^{m-b}$, satisfying the following: 
    \begin{enumerate}
    \item $\{\sS_1,\ldots,\sS_{2^b}\}$ is a partition of $\{0,1\}^m$. 
    \item The vectors $\sS_{1,1},\ldots, \sS_{2^b,1}$ share the same values in the last $m-b$ bits. 
    \item The vector $(0,\ldots, 0)$ is the last element $\sS_{1,2^{m-b}}$ of the first sequence. 
    \item $\forall i\in\{1,\ldots,2^b\}$, $\forall k\in\{1,\ldots,2^{m-b} -1\}$ we have $\|\sS_{i,k},\sS_{i,k+1}\|_H = 1$. 
    \item $\forall i,r\in\{1,\ldots,2^b\}$ such that $i\neq r$ and $\forall k\in\{1,\ldots, 2^{m-b} -1\}$, the bit switched between $\sS_{i,k}$ and $\sS_{i,k+1}$ and the bit switched between $\sS_{r,k}$ and $\sS_{r,k+1}$ are different, unless $\|S_{i,k} - S_{r,k}\|_H = 1$. 
    \end{enumerate}
\end{theorem}

This theorem describes a schedule that allows for probability to be shared off of $2(m-b)$ vectors per layer, starting from the vectors $\sS_{i,1}$, $i=1,\ldots, 2(m-b)$. 
For every layer $l$, if ${\bf h}^{l-1}$ matches $\sS_{i,l-1}$, probability mass will be shared onto $\sS_{i,l}$, for each $i=1,\ldots, 2(m-b)$, 
and it will be copied unchanged otherwise. 
At this, the accuracy of the transition probabilities depends on the maximum allowed magnitude of the weighs and biases, similar to equation~\eqref{eq:sharingstep}.

\subsection{Universal Approximation}
\label{section:proofdeep}

\subsubsection*{The First Layer}

The first step of the proof focuses on the flexibility of the first layer of the network. 
For fixed $d,s,j,$ there are $2^j(s+d-j)$ units in ${\bf h}^1$ belonging to $2^j$ consecutive blocks, indexed by $\t=1,\ldots, 2^j$. 

Within each block, set the parameters according to Proposition \ref{prop:sharing1} 
to copy the last $d-j$ bits of the input ${\bf x}$, 
so that ${\bf b}^1_{(\t)} = {\bf x}_{[j+1,d]}$ 
with probability $(1-\e)^{j-d}$ and parameters of magnitude no more than $\alpha=2\sigma^{-1}(1-\e)$. 

Within each block, ${\bf a}_{(\t)}^1$ will activate with probability close to one for exactly $2^{d-j}$ inputs ${\bf x}$. 
This can be done by setting the parameters of the first unit ${\bf a}_{(\t),1}^1$ of each block such that it takes state $1$ only if the first $j$ bits of ${\bf x}$ agree with the number $\t$ of the block. This is formalized in the following lemma. 
The remainder units ${\bf a}_{(\t),i}^1$, $i=2,\ldots, s$ are set to take the zero state with probability $(1-\e)$, by choosing their weights to be zero and bias $-\sigma^{-1}(1-\e)$. 

\begin{lemma} 
\label{lemma:deep1}
    Fix $\t \in \{1,2,\dots,2^j\}$ and $\e \in (0,1/2)$.  Let $S = \big\{ {\bf x} \in \{0,1\}^d \; \Big| \; \big\lfloor{\frac{\operatorname{int}({\bf x})}{2^{d-j}}\big\rfloor} + 1 = \t \big\}$. 
    Then there exist weights ${\bf W} \in \R^{1\times d}$ and bias $b \in \R$ for the first unit of the $\t$-th block, 
    having absolute values at most $\alpha=2d\sigma^{-1}(1-\epsilon)$,
    such that 
    \begin{equation}
	    \Pr({\bf a}^1_{(\t),1} = 1 | {\bf x}) = 
	    \begin{cases} 
		    (1-\e) & {\bf x} \in S \\ \e & {\bf x} \notin S
	    \end{cases}. 
    \end{equation}
\end{lemma}

\begin{proof} The probability of unit ${\bf a}^1_{(\t),1}$ activating is given by 
\begin{equation}
	\Pr({\bf a}^1_{(\t),1} = 1 | {\bf x}) = \s({\bf W}{\bf x} + b). 
\end{equation}
Note that $S$ is the set of length-$d$ binary vectors whose first $j$ bits equal the length-$j$ binary vector ${\bf g}$ with integer representation $\t$. 
Geometrically, $S$ is a $d-j$ dimensional face of the $d$-hypercube. 
In turn, there exists an affine hyperplane in $\mathbb{R}^d$ separating $S$ from $\{0,1\}^d\setminus S$. 
For instance, we may choose ${\bf W} = \gamma(2({\bf g}-\frac{1}{2}), {\bf 0})^\top$ and $b=\gamma (\|{\bf g}\|_1-\frac{1}{2})$, which gives ${\bf W}{\bf x} + b = \frac12 \gamma$ for all ${\bf x}\in S$, 
and ${\bf W}{\bf x} + b \leq -\frac12\gamma$ for all ${\bf}\not\in S$. 
Choosing $\gamma = 2\sigma^{-1}(1-\epsilon)$ yields the claim. 
\end{proof}

 Note that ${\bf x}\in S$ is equivalent to ${\bf x}_{[1,j]}={\bf g}$, where ${\bf g}$ is the $j$ bit representation of $\t$. Following \eqref{eq:sharingstep_a}--\eqref{eq:sharingstep} we can also have the second bit activate as 
     \begin{equation}
 	    \Pr({\bf a}^1_{(\t),2} = 1 | {\bf x}) = 
 	    \begin{cases} 
 		    \rho_\e, & {\bf x}_{[1,j]}={\bf g} \\ 
 		    \e, & {\bf x}_{[1,j]}\neq{\bf g}
 	    \end{cases} , 
 	    \label{eq:save-one-layer}
     \end{equation}
for any chosen $\rho\in[0,1]$. 
We will be able to use this type of initialization to save one layer when $j=d$, where there is only one subsection per section. 

\subsubsection*{The Hidden Layers}

In the second part of our construction, the focus is restricted to individual sections of the network, having width $s+(d-j)$ and $L = 2^{d-j}(2^{s-b}+2(s-b)-1)$ layers. 
To prevent separate sections from interfering with one another, choose all weights between units in sections $\t$ and $\t'$ to zero, for all $\t' \neq \t$. 
The $\t$-th section will only be contingent upon its parameters and ${\bf h}_{(\t)}^1$, which can be regarded as the input to the section. 
Each section will be responsible for approximating the target conditional distributions of $2^{d-j}$ inputs. 
Each section should be thought of as consisting of $2^{d-j}$ subsections in sequence, each consisting of $2^{s-b}+2(s-b)-1$ consecutive layers. 

Each subsection will be responsible for approximating the target conditional distribution of a single input ${\bf x}$. 
The first layer of any subsection copies the state from the previous layer, except the very first subsection, which we already described above. 
Subsection $q$ will be ``activated'' if ${\bf a}_{(\t)}^l = (1,0,\ldots,0)$ and ${\bf b}^l_{(\tau)}$ takes the specific value $\operatorname{bin}_{d-j}(q)$, where $l$ is the first layer of the subsection. 
When a subsection is activated, it will carry out a sequence of sharing steps to generate the output distribution. 
This can be achieved in $2^s$ layers by applying a single sharing step per layer, with schedule given by a Gray code with initial state $(1,0,\ldots, 0)$ and final state $(0,\ldots,0)$.  If only single sharing steps are used, then the parameters which need to be trainable are biases.
Alternatively, we can have the first $2(s-b)$ layers of the subsection conduct probability sharing to distribute the mass of ${\bf a}_{(\t)}^l = (1,0,\ldots,0)$ across the initial states $\sS_{i,1}$, $i=1,\ldots, 2(s-b)$, of the partial Gray codes from Theorem~\ref{theorem:sharing2}. Following this, the subsection overlays the $2^b=2(s-b)$ sequences of sharing steps with the schedule from Theorem~\ref{theorem:sharing2}, to generate the output distribution. 
When the subsection is not activated, it copies whatever incoming vector downwards until its last layer. 

By the construction one can see that if ${\bf a}_{(\t)}^1 = {\bf 0}$, probability mass never is transferred off of this state, 
meaning the last hidden layer of the section takes state ${\bf a}_{(\t)}^{L} = {\bf 0}$ with probability close to one. 
Therefore the blocks of the final hidden layer will be distributed as 
\begin{equation}
    {\bf a}_{(\t)}^{L} \sim 
    \begin{cases} 
        p^*(\cdot | {\bf x}), & \text{if } \big(\lfloor{\frac{\operatorname{int}({\bf x})}{2^{d-j}}\rfloor}+1\big) = \t \\ 
        \d_{\bf 0}, & \text{otherwise}
    \end{cases}.
\end{equation}

We can obtain a slight improvement when $j = d$ and there is only one subsection per section.  In this case, we can set two of the initial states $\sS_{i,1}$, $i=1,\ldots, 2(s-b)$, as $(1,0,0,\ldots, 0)$ and $(1,1,0,\ldots, 0)$. 
We initialize the first hidden layer by \eqref{eq:save-one-layer}, which allows us to place probabilities $\rho_\e$ and $(1-\rho_\e)$ on these two states and so, to save one of the $2(s-b)$ sharing steps that are used to initialize the partial Gray codes. 

\subsubsection*{The Output Layer}

The third and final step of the proof is to specify how the last layer of the network will copy the relevant block of the final hidden layer such that the output layer exhibits the correct distribution. 
To this end, we just need the $i$-th unit of the output layer to implement an \texttt{or} gate over the $i$-th bits of all blocks, for $i=1,\ldots,s$. 
This can be achieved, with probability $(1-\e)$, by setting the 
the bias of each output unit as $-\alpha$, 
and the weight matrix ${\bf W} \in \R^{s \times 2^j (s+d-j)}$ 
of the output layer as 
\begin{equation}
    {\bf W} = 
    \a \bpm I_s & | & Z & | & I_s & | & Z & | & \cdots & | & I_s & | & Z \epm,
\end{equation}
where $I_s$ is the $s\times s$ identity matrix, $Z$ is the $s \times (d-j)$ zero matrix, and $\alpha= 2 \sigma^{-1}(1-\e)$. 
This concludes the proof of Theorem~\ref{theorem:deep2}. 

\begin{example}
Consider the case where $d = 1$ and $s = 2$ and the simple sharing scheme by a Gray code.  The narrowest case where $j = 0$ is a network consisting of $2\cdot 2^2 = 8$ layers, each of width $2+1$. There is only one section, and it consists of two subsections, one for each possible input. 
The first subsection is responsible for approximating $p^*(\cdot | {\bf x} = 0)$ and the second for $p^*(\cdot | {\bf x} = 1)$. 
See Figure~\ref{fig:example1} for an illustration of the network and the Gray codes used to specify sharing steps. 
\end{example}

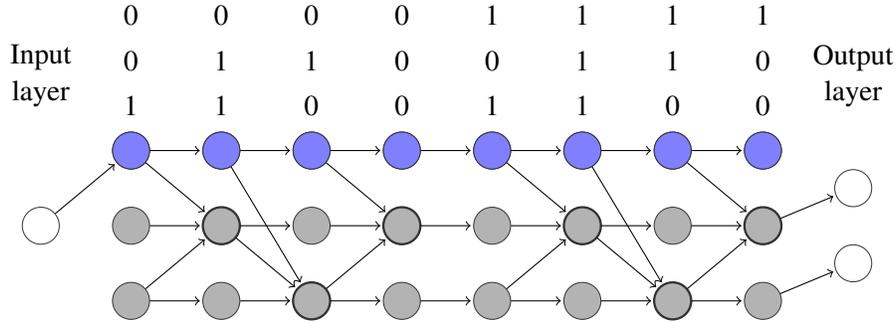
\begin{figure}[ht]
\centering
\begin{tikzpicture}[shorten >=1pt,->,draw=black!40, node distance=\layersep]
    \tikzstyle{every pin edge}=[<-,shorten <=1pt]
    \tikzstyle{neuron}=[circle,fill=black!25,minimum size=14pt,inner sep=0pt]
    \tikzstyle{input neuron}=[neuron, draw=black!80, line width=0.1mm, fill = black!0!white];
    \tikzstyle{output neuron}=[neuron, draw=black!80, line width=0.1mm, fill = black!0!white];
    \tikzstyle{hidden neuron}=[neuron, draw=black!80, line width=0.1mm, fill = black!30!white];
    \tikzstyle{hidden neuronz}=[neuron, draw=black!80, line width=0.1mm, fill = white!50!blue];
    \tikzstyle{Thick neuron}=[neuron, draw=black!80, line width=0.3mm, fill = black!30!white];

    \tikzstyle{annot} = [text width=4em, text centered]

    \node[input neuron] (I-1) at (-6,-2) {};

    \foreach \name / \y in {1,...,3}
        \path node[hidden neuron] (H1-\name) at (-4.8,-\y cm) {};
    \foreach \name / \y in {1,...,3}
        \path node[hidden neuron] (H2-\name) at (-3.6,-\y cm) {};
    \foreach \name / \y in {1,...,3}
        \path node[hidden neuron] (H3-\name) at (-2.4,-\y cm) {};
    \foreach \name / \y in {1,...,3}
        \path node[hidden neuron] (H4-\name) at (-1.2,-\y cm) {};
    \foreach \name / \y in {1,...,3}
        \path node[hidden neuron] (H5-\name) at (0,-\y cm) {};
    \foreach \name / \y in {1,...,3}
        \path node[hidden neuron] (H6-\name) at (1.2,-\y cm) {};
    \foreach \name / \y in {1,...,3}
        \path node[hidden neuron] (H7-\name) at (2.4,-\y cm) {};
    \foreach \name / \y in {1,...,3}
        \path node[hidden neuron] (H8-\name) at (3.6,-\y cm) {};

    \node[Thick neuron] (T1) at (-3.6,-2cm) {};
    \node[Thick neuron] (T2) at (-1.2,-2cm) {};
    \node[Thick neuron] (T3) at (1.2,-2cm) {};
    \node[Thick neuron] (T4) at (-2.4,-3cm) {};
    \node[Thick neuron] (T5) at (2.4,-3cm) {};
    \node[Thick neuron] (T6) at (3.6,-2cm) {};
    
    \foreach \name / \x in {4.8,3.6,2.4,1.2,0,-1.2,-2.4,-3.6}
        \path node[hidden neuronz] (B-\name) at (-\x,-1 cm) {};

    \foreach \name / \y in {1,...,2}
        \node[output neuron] (O-\name) at (4.8,-\y-0.5) {};

    \path (H1-1) edge[black] (H2-1);
    \path (H1-1) edge[black] (H2-2);
    \path (H1-2) edge[black] (H2-2);
    \path (H1-3) edge[black] (H2-2);
    \path (H1-3) edge[black] (H2-3);

    \path (H3-1) edge[black] (H4-1);
    \path (H3-1) edge[black] (H4-2);
    \path (H3-2) edge[black] (H4-2);
    \path (H3-3) edge[black] (H4-2);
    \path (H3-3) edge[black] (H4-3);

    \path (H5-1) edge[black] (H6-1);
    \path (H5-1) edge[black] (H6-2);
    \path (H5-2) edge[black] (H6-2);
    \path (H5-3) edge[black] (H6-2);
    \path (H5-3) edge[black] (H6-3);

    \path (H7-1) edge[black] (H8-1);
    \path (H7-1) edge[black] (H8-2);
    \path (H7-2) edge[black] (H8-2);
    \path (H7-3) edge[black] (H8-2);
    \path (H7-3) edge[black] (H8-3);

    \path (H2-1) edge[black] (H3-1);
    \path (H2-1) edge[black] (H3-3);
    \path (H2-2) edge[black] (H3-2);
    \path (H2-2) edge[black] (H3-3);
    \path (H2-3) edge[black] (H3-3);

    \path (H4-1) edge[black] (H5-1);
    \path (H4-2) edge[black] (H5-2);
    \path (H4-3) edge[black] (H5-3);

    \path (H6-1) edge[black] (H7-1);
    \path (H6-1) edge[black] (H7-3);
    \path (H6-2) edge[black] (H7-2);
    \path (H6-2) edge[black] (H7-3);
    \path (H6-3) edge[black] (H7-3);

    \path (H8-3) edge[black] (O-2);
    \path (H8-2) edge[black] (O-1);

    \path (I-1) edge[black] (H1-1);

    \node[annot,above of=I-1, node distance=2cm] (il) {Input layer};
    \node[annot,above of=O-1, node distance=1.5cm] {Output layer};

    \node[annot,above of=H1-1, node distance=0.6cm] (a1) {1};
    \node[annot,above of=a1, node distance=0.6cm] (a2) {0};
    \node[annot,above of=a2, node distance=0.6cm] (a3) {0};

    \node[annot,above of=H2-1, node distance=0.6cm] (b1) {1};
    \node[annot,above of=b1, node distance=0.6cm] (b2) {1};
    \node[annot,above of=b2, node distance=0.6cm] (b3) {0};

    \node[annot,above of=H3-1, node distance=0.6cm] (c1) {0};
    \node[annot,above of=c1, node distance=0.6cm] (c2) {1};
    \node[annot,above of=c2, node distance=0.6cm] (c3) {0};

    \node[annot,above of=H4-1, node distance=0.6cm] (d1) {0};
    \node[annot,above of=d1, node distance=0.6cm] (d2) {0};
    \node[annot,above of=d2, node distance=0.6cm] (d3) {0};

    \node[annot,above of=H5-1, node distance=0.6cm] (e1) {1};
    \node[annot,above of=e1, node distance=0.6cm] (e2) {0};
    \node[annot,above of=e2, node distance=0.6cm] (e3) {1};

    \node[annot,above of=H6-1, node distance=0.6cm] (f1) {1};
    \node[annot,above of=f1, node distance=0.6cm] (f2) {1};
    \node[annot,above of=f2, node distance=0.6cm] (f3) {1};

    \node[annot,above of=H7-1, node distance=0.6cm] (g1) {0};
    \node[annot,above of=g1, node distance=0.6cm] (g2) {1};
    \node[annot,above of=g2, node distance=0.6cm] (g3) {1};

    \node[annot,above of=H8-1, node distance=0.6cm] (h1) {0};
    \node[annot,above of=h1, node distance=0.6cm] (h2) {0};
    \node[annot,above of=h2, node distance=0.6cm] (h3) {1};

\end{tikzpicture}
\caption{The narrow network topology where $d = 1, \;s = 2,\; j = 0$, where only one probability sharing step is utilized per layer, and the Gray codes specifying them are shown.  This network is a universal approximator of Markov kernels in $\D_{1,2}.$  The units shown in blue will copy the input state throughout the network.  The first hidden layer of the network will either be $``100"$ or $``101"$ depending on the input.  If it is the former, only the first subsection ${\bf h}^{[1,4]}$ will perform probability mass sharing, resulting in ${\bf h}^4$ distributed as $p^*(\cdot | x = 0)$.  If it is the latter, probability mass sharing will only occur in ${\bf h}^{[5,8]}$. Only the non-zero weights are shown.  This network only requires 6 trainable parameters which are all biases, indicated by the thick-outlined units.} 
\label{fig:example1}
\end{figure}

\subsection{Error Analysis for Finite Weights and Biases}
\label{section:errr}

The proof construction demonstrates that we may conduct sequential probability sharing steps according to Gray codes which specify a unique path from every ${\bf x}$ to every ${\bf y}$. 
Given an input ${\bf x}$, the path generating ${\bf y}$ will occur with probability $p^*({\bf y} | {\bf x})$ as the size of the parameters becomes infinitely large. 
If the magnitude of the parameters is bounded by $\a = 2m(\s^{-1}(1-\e))$ and we consider a target kernel $P \in \D^\e_{d,s} := \{P \in \D_{d,s} \colon \e \leq P_{ij} \leq 1-\e \text{ for all } i,j\}$, we may utilize Theorem~\ref{theorem:sharing1} to compute bounds on the probability of the path that is intended to generate ${\bf y}$. 
Any other paths from {\bf x} to {\bf y} will have a low probability. The details are as follows. 

\begin{proof}[Proof of Theorem~\ref{thrm:errr}]
Fix an $\e \in (0,\frac{1}{2^s})$ and $p^\ast \in \D_{d,s}^\e$. 
Suppose that a network from Theorem~\ref{theorem:deep2} has $L$ hidden layers of width $m$.  Without loss of generality, assume that ${\bf y} = {\bf g}^r$ for some $r \in \{1,\dots,2^s\}$ where the sequence $\{{\bf g}^l\}$ is the Gray code defining the sharing steps.
Recall that $p({\bf y}|{\bf x})$ may be written as
\begin{align}
p({\bf y} | {\bf x}) &= \sum_{{\bf h}^1} \cdots \sum_{{\bf h}^L}\; p({\bf y} | {\bf h}^L)\cdots p({\bf h}^1 | {\bf x})\\
&= \sum_{{\bf h}} \; p({\bf y} | {\bf h}^L)\cdots p({\bf h}^1 | {\bf x}).
\end{align}
Note that most terms in this sum are $\sO(\e)$ or smaller when using the proof construction for Theorem~\ref{theorem:deep2}.  The one term that is larger than the rest is the term where the hidden layers activate as the sequence ${\bf h}^1 = {\bf g}^1,\; \dots, \;{\bf h}^r = {\bf g}^r, \; {\bf h}^{r+1} = {\bf g}^r, \; \dots, \; {\bf h}^L = {\bf g}^r$.  In particular, if the parameters in the network were infinitely large, this sequence of hidden layer activations would occur with exactly probability $p^*({\bf y}|{\bf x})$ by construction.  Denote this sequence by $T$ and let $p({\bf y},T | {\bf x})$ denote the probability of observing this sequence, $p({\bf y},T | {\bf x}) := p({\bf y} | {\bf h}^L = {\bf g}^r)p({\bf h}^{L} = {\bf g}^r | {\bf h}^{L-1} = {\bf g}^r)\cdots p({\bf h}^r = {\bf g}^r | {\bf h}^{r-1} = {\bf g}^{r-1})\cdots p({\bf h}^1 = {\bf g}^1 | {\bf x})$. 

When the magnitude of the weights is bounded by $2m\s^{-1}(1-\e)$, Theorem~\ref{theorem:sharing1} provides the error terms for each $p({\bf h}^l = {\bf g}^l | {\bf h}^{l-1} = {\bf g}^{l-1})$. Specifically, we have that
\begin{align*}
p({\bf h}^1 = {\bf g}^1 | {\bf x}) &= (1-\e)^m\\
p({\bf h}^2 = {\bf g}^2 | {\bf h}^1 = {\bf g}^1) &= \rho^{[1]}(1-\e)^{m-1}\\
\vdots\hspace{1cm} & \hspace{1cm} \vdots\\
p({\bf h}^r = {\bf g}^r | {\bf h}^{r-1} = {\bf g}^{r-1}) &= \rho^{[r-1]}(1-\e)^{m-1}\\
p({\bf h}^{r+1} = {\bf g}^r | {\bf h}^r = {\bf g}^r) &= (1-\rho^{[r]})(1-\e)^{m-1}\\
p({\bf h}^{r+2} = {\bf g}^r | {\bf h}^{r+1} = {\bf g}^r) &= (1-\e)^{m}\\
\vdots\hspace{1cm} & \hspace{1cm} \vdots\\
p({\bf h}^{L} = {\bf g}^r | {\bf h}^{L-1} = {\bf g}^r) &= (1-\e)^{m}\\
p({\bf y} | {\bf h}^L = {\bf g}^r) &= (1-\e)^{s}
\end{align*}
where $\rho^{[l]}$ are the transfer probabilities between layers $l$ and $l+1$ discussed in Theorem~\ref{theorem:sharing1}.  We point out that for the output of the network to be ${\bf y} = {\bf g}^r$, the complementary sharing probability must occur at layer $l = r$, i.e., $(1-\rho^{[r]})$.  Additionally, we point out that $p^*({\bf y}|{\bf x}) = \rho^{[1]}\rho^{[2]}\cdots\rho^{[r-1]}(1-\rho^{[r]})$.  With this, the bound for Theorem~\ref{thrm:errr} may be shown as follows:
\begin{align*}
|p({\bf y} | {\bf x}) - p^*({\bf y} | {\bf x})| 
&= \bigg|\sum_{{\bf h}^1} \cdots \sum_{{\bf h}^L} \; p({\bf y} | {\bf h}^L)\cdots p({\bf h}^1 | {\bf x}) - p^*({\bf y} | {\bf x})\bigg|\\
&\leq |p(T,{\bf y} | {\bf x}) - p^*({\bf y} | {\bf x})| 
+ \bigg|\sum_{({\bf h}^1,\ldots,{\bf h}^L) \neq T} \; p({\bf y} | {\bf h}^L)\cdots p({\bf h}^1 | {\bf x})\bigg|\\
&< |p(T,{\bf y} | {\bf x}) - p^*({\bf y} | {\bf x})| + \e\\ 
&= |\r^{[1]}\cdots \r^{[r-1]}(1-\r^{[r]})(1-\e)^{mL-r+s} - p^*({\bf y} | {\bf x})| + \e\\
&= |p^*({\bf y} | {\bf x})(1-\e)^{mL-r+s} - p^*({\bf y} | {\bf x})| + \e\\
&= p^*({\bf y} | {\bf x})|1 - (1-\e)^{mL-r+s}| + \e\\
&< 1 - (1-\e)^{mL-r+s} + \e.
\end{align*}
In the third line, the second term is upper bounded by $\e$ because each term in the sum has at least one factor of $\e$, and the sum itself can not be larger than $1$.  
Since $mL-r + s \leq N$, the total number of units in the network excluding the input units, for any $r \in 2^s$, we can uniformly bound the difference in each probability using $1 - (1-\e)^{N}$. 

It remains to show that if $p^\ast\in\Delta_{d,s}^\e$, then for each ${\bf x}$ the factorization $p^*({\bf g}^r|{\bf x}) = \rho^{[1]}\cdots\rho^{[r-1]}(1-\rho^{[r]})$ has factors in $[\e,1-\e]$ for each ${\bf g}^r$. 
Since $p^*({\bf g}^1|{\bf x})=(1-\rho^{[1]})\geq\e$, we have that $\rho^{[1]}\leq 1-\e$. 
Similarly, since $p^*({\bf g}^2|{\bf x})= \rho^{[1]}(1-\rho^{[2]}) \geq\e$ and $(1-\rho^{[2]})\leq 1$, we have that $\rho^{[1]}\geq\e$. The same argument applies recursively for all $r$. 

Finally, for an arbitrary target kernel $p^\ast\in\Delta_{d,s}$ one finds an approximation $p^{\ast,\e}\in\Delta^\e_{d,s}$ with 
$|p^\ast({\bf y| {\bf x}}) - p^{\ast,\e}({\bf y| {\bf x}})|\leq \e$ 
and 
$|p({\bf y| {\bf x}}) - p^\ast({\bf y| {\bf x}})|\leq |p({\bf y| {\bf x}}) - p^{\ast,\e}({\bf y| {\bf x}})| + \e$. 
\end{proof}

\subsection{Discussion of the Proofs for Deep Networks}

Since universal approximation was shown for the shallow case, it follows that any stochastic feedforward network of width at least $2^{d}(2^{s-1}-1)$ with $s \geq 2, \; d \geq 1$ and $L \geq 1$ hidden layers is a universal approximator.  The proof above refines this bound by showing that as a network is made deeper, it may be made proportionally narrower while still being a universal approximator.  This proof applies to network topologies indexed by $j \in \{0,1,\dots,d\}$, where the shallowest depth occurs when $j = d$, where $2^{s-b} + 2(s-b)$ 
hidden layers is sufficient where $b \sim \log_2(s)$. 
This leaves open whether or not there is a spectrum of networks between the $j = d$ case and the shallow case which are also universal approximators.  Beyond the narrow side of the spectrum where $j = 0$, it is also open whether or not narrower universal approximators exist. 
This is due to the proof technique which relies on information about the input being passed from layer to layer. 

We point out that universal approximation of $\Delta_{d,s}$ requires unbounded parameters. Indeed, if we want to express a conditional probability value of $1$ as the product of conditional probabilities expressed by the network, then some of these factors need to have entries $1$. On the other hand it is clear that a sigmoid unit only approaches value $1$ as its total input tends to infinity, which requires that the parameters tend to infinity. 
We provided bounds on the approximation errors when the parameters are bounded in magnitude. 
However, it is left open whether universal approximation of kernels in $\Delta_{d,s}^\e$, with entries bounded away from $0$ and $1$, is possible with finite weights. 
This is because our proof technique relies on inducing nearly-deterministic behavior in many of the computational units by sending the weights toward infinity. 
Nonetheless, as shown in our bounds and illustrated in Section~\ref{sec:examples}, most of the approximation quality is already present for moderately sized weights. 

The networks that we discussed here are optimal in the sense that they only utilize $2^d(2^s-1)$ trainable parameters. This follows from the observation that each probability mass sharing step has exactly one parameter that depends on the kernel that is being approximated. 
Further improvements on finding more compact universal approximators in the sense of having less units may be possible. 
It remains of interest to determine the tightness of our theorems in the sense of the number of units. 
Lower bounds for the width and overall number of parameters of the network can be determined by information theoretic and geometric techniques discussed next.

\section{Lower Bounds for Shallow and Deep Networks}
\label{sec:lowerbounds}

\subsection*{Parameter Counting Lower Bounds}

The following theorem establishes a lower bound on the number of parameters needed in a network for universal approximation to be possible. 
It verifies the intuition that the number of trainable parameters of the model needs to be at least as large as the dimension of the set of kernels that we want to approximate arbitrarily well.  This result is needed in order to exclude the possibility of a space-filling curve type of lower dimensional universal approximator. 

The proof is based on finding a smooth parametrization of the closure of the model and then applying Sard's theorem \citep{sard1942}. 
We start with the parametrization. 

\begin{proposition}
\label{proposition:closed-parametrization}
The closure of the set of kernels $F_{d,m_1,\ldots, m_L,s}\subseteq\Delta_{d,s}$ represented by any Bayesian sigmoid belief network can be parametrized in terms of a finite collection of smooth maps with compact parameter space of the same dimension as the usual parameter space. 
\end{proposition}

\begin{proof}[Proof of Proposition~\ref{proposition:closed-parametrization}]
Since compositions of units to create a network corresponds to taking matrix products and marginalization corresponds to adding entries of a matrix, both of which are smooth maps, it suffices to prove the statement for a single unit. 
Consider the set $F_{m,1}$ of kernels in $\Delta_{m,1}$ represented by a unit with $m$ binary inputs. 
The usual parametrization takes $w = (w_0,w_1,\ldots, w_m)\in \mathbb{R}^{m+1}$ to the kernel in $\Delta_{m,1}$ given by 
the $2\times 2^m$ matrix 
\begin{equation*}
[\sigma(\sum_{j=0}^m w_j h_j ), \sigma(-\sum_{j=0}^m w_j h_j)]_{h\in\{1\}\times\{0,1\}^m}, 
\end{equation*}
where any $h=(h_0,h_1,\ldots,h_m)$ has first entry $h_0=1$. 
We split the parameter space $\mathbb{R}^{m+1}$ into the $2^{m+1}$ closed orthants. 
Fix one of the orthants $\mathbb{R}^{m+1}_S$, which is specified by a partition of 
$\{0,1,\ldots,m\}$ into a set $S$ of coordinates that are allowed to be negative, and the complementary set $L$ of coordinates that are allowed to be positive. 
Now consider the bijection $w\in \mathbb{R}^{m+1}_S\to [\omega, \gamma]\in (0,1]^{m+1}$ with $\omega_j = \exp(w_j)$ for each $j\in S$, and $\gamma_j = \exp(-w_j)$ for each $j\not\in S$. 
Then 
\begin{align*}
\sigma(\sum_{j=0}^m w_j h_j) 
=& \frac{\exp(\sum_{j=0}^m w_j h_j)}{\exp(\sum_{j=0}^m w_j h_j) + 1}  \\
=& \frac{\prod_{j=0}^m\exp(w_j h_j)}{\prod_{j=0}^m\exp(w_j h_j) + 1} \\
=& \frac{\prod_{j\in S} \exp(w_j h_j)}{\prod_{j\in S}\exp(w_j h_j) + \prod_{j\not\in S}\exp(-w_j h_j)} \\
=& 
\frac{\prod_{j\in S} \omega_j^{h_j}}{\prod_{j\in S}\omega_j^{h_j} + \prod_{j\not\in S}\gamma_j^{h_j}} . 
\end{align*}
This defines a smooth map $\psi_S \colon [0,1]^{m+1}\to\Delta_{m,1}$ (or a finite family of smooth maps over the relative interiors of the faces of $[0,1]^{m+1}$). 
Since the map $\psi_S$ is continuous and its domain is compact and its co-domain is Hausdorff, the image $\psi_S([0,1]^{m+1})$ is closed. 
In turn, the union over different orthants $\cup_S \psi_S([0,1]^{m+1})$ is a closed set which contains $F_{m,1}$ as a dense subset, so that it is equal to $\overline{F_{m,1}}$. 
\end{proof}

\begin{theorem}\label{thm:counting}
Consider a stochastic feedforward network with $d$ binary inputs and $s$ binary outputs.  
If the network is a universal approximator of Markov kernels in $\Delta_{d,s}$, then necessarily the  number of trainable parameters is at least $2^d(2^s-1)$. 
\end{theorem}
The space of Markov kernels is $\Delta_{d,s} = \D_s \times \cdots \times \D_s$ ($2^d$ times), and has dimension $2^d(2^s-1).$  This theorem states that at least one parameter is needed per degree of freedom of a Markov kernel. 

\begin{proof}[Proof of Theorem~\ref{thm:counting}] 
Consider one of the smooth and closed maps $\psi$ provided in Proposition~\ref{proposition:closed-parametrization} and denote its input space by $\Omega=[0,1]^k$. 
Sard's theorem states that the set of critical values of a smooth map is a null set. 
If the input-space dimension is less than the output-space dimension, then every point is a critical point and the image of the map is a null set. 
Therefore, we conclude that if $\dim(\Omega)=k$ is less than $2^d(2^s-1)$, the set $\psi(\Omega) = \overline{\psi(\Omega)}$ is a null set. Since the closure of the model is a finite union of such sets, it cannot possibly be a universal approximator if the dimension is of the parameter space is less than indicated. 
\end{proof}

\subsection*{Minimum Width}

A universal approximator can not have too narrow layers. 
We can show this by utilizing the data processing inequality. 
Another approach is in terms of the combinatorics of the tuples of factorizing distributions represented by a layer of stochsastic units. 

We start with the approach based on the data processing inequality. 
To be precise, consider the mutual information of two discrete random vectors ${\bf X}$ and ${\bf Y}$, which is defined as 
\begin{equation}
MI({\bf X};{\bf Y}) = 
H({\bf X}) -H({\bf X}|{\bf Y}) 
= 
H({\bf Y}) -H({\bf Y}|{\bf X}),  
\end{equation}
where $H({\bf Y}) = -\sum_{\bf y} p({\bf y}) \log p({\bf y})$ stands for the entropy of the probability distribution of ${\bf Y}$, 
and $H({\bf Y}|{\bf X}) = -\sum_{\bf x} p({\bf x}) \sum_{\bf y} p({\bf y}|{\bf x}) \log p({\bf y} | {\bf x})$ stands for the conditional entropy of ${\bf Y}$ given ${\bf X}$. 
If the state spaces are $\mathcal{X}$ and $\mathcal{Y}$, then the maximum value of the mutual information is  $\min\{\log |\mathcal{X}|, \log|\mathcal{Y}|\}$. This value is attained by any joint distribution for which one of the variables is uniformly distributed and its state is fully determined by the observation of the other variable. 

The data processing inequality states that if a joint distribution satisfies the Markov chain $p({\bf x},{\bf h}, {\bf y}) = p({\bf x}) p({\bf h}|{\bf x}) p({\bf y}|{\bf h})$, then the mutual information behaves monotonically in the sense that $MI({\bf X};{\bf Y}) \leq MI({\bf X};{\bf H})$. 
Note that this inequality is independent of how the conditional distributions are parametrized, and in special cases there might exist stronger inequalities. 
From the generic inequality given above we infer the following. 

\begin{proposition} \label{prop:narrow1}
Consider a sigmoid stochastic feedforward network with $d$ inputs and $s$ outputs.  If the network is a universal approximator of Markov kernels in $\D_{d,s}$, then each hidden layer has at least $\min\{d,s\}$ units. 
\end{proposition}

 \begin{proof}
 The network is a universal approximator of Markov kernels if and only if the model augmented to include arbitrary probability distributions over the inputs is a universal approximator of joint distributions over inputs and outputs. 
 In view of the data processing inequality, if any of the hidden layers has less than $\min\{d,s\}$ units, then the joint distributions of inputs and outputs represented by the network satisfy non-trivial inequalities of the mutual information, meaning that an open nontrivial set of joint distributions is excluded. 
 \end{proof}

We can further strengthen this result in the case where the number of inputs is smaller than the number of outputs. 

\begin{proposition} \label{prop:narrow2}
Consider stochastic feedfoward networks with $d$ inputs and $s > d$ outputs. 
\begin{itemize}
\item 
    If $d \geq 0$ and $s\geq 2$, the last hidden layer of a universal approximator has at least $s - 1$ units when $s$ is even, and at least $s$ units when $s$ is odd. 
\item 
    If $d \geq 1$, the last hidden layer of a universal approximator has at least $s$ units. 
\end{itemize}
\end{proposition}

\begin{proof}
Note that the output distribution is a mixture of the conditional distributions of the output layer given all possible values of the second last layer, all of which are factorizing distributions. 
Further, note that if a factorizing distribution has support strictly contained in the set of even (or odd) parity strings, then it must be a point measure. 

Consider $d\geq0$ and, as a desired output distribution, the uniform distribution on strings of even parity. 
In order for the the network to represent this, the last kernel needs to contain the $2^{s-1}$ point measures on even parity strings as rows. 
In turn, the last hidden layer must have at least $s-1$ units. 
The lower bound $s$ results from the fact that the rows of the kernels are not arbitrary product distributions. 
Indeed, for a module with $m$ input units, the $2^m$ rows of the kernel are factorizing distributions with shared parameters of the form ${\bf W}{\bf h} + {\bf b}$, ${\bf h}\in\{0,1\}^m$. 
The parameter vector of a point measure that is concentrated on a given vector ${\bf y}\in\{0,1\}^s$ is a vector on the $y$-th orthant of $\mathbb{R}^s$. 
The set of parameters ${\bf W}{\bf h} + {\bf b}$, ${\bf h}\in\{0,1\}^m$ intersects all even parity orthants of $\mathbb{R}^s$ only if $m \geq s$ \citep[see][Proposition~3.19]{doi:10.1137/140957081}. 

Now consider $d \geq 1$ and, as a desired pair of output distributions for two different inputs, a distribution supported strictly on the even parity strings and a distribution supported on the odd parity strings. This requires that the last kernel has all $2^s$ point measures as rows, and hence at least $s$ inputs. 
\end{proof}
An example of Proposition~\ref{prop:narrow2} is shown in Figure~\ref{fig:my_label1}. 

\begin{figure}
    \centering
    \begin{tikzpicture}
    \node[circle, fill = gray!20, inner sep=0pt, minimum size=.5cm, draw=black, label=center:${}$] (X1) at (1, 2) {}; 
    \foreach \x in {0,1,2,3,4} { 
    \node[circle, fill = gray!20, inner sep=0pt, minimum size=.5cm, draw=black, label=center:${}$] (X2\x) at (\x-1, 1) {};
	\draw[->,shorten >= 2pt, shorten <= 2pt] (X1) -- (X2\x);  
    }
    \foreach \y in {0,1,2} { 
    \node[circle, fill = gray!20, inner sep=0pt, minimum size=.5cm, draw=black, label=center:${}$] (X3\y) at (\y, 0) {};
        \foreach \x in {0,1,2,3,4} { 
	    \draw[->,shorten >= 1pt, shorten <= 2pt] (X2\x) -- (X3\y);   }
    }
    \foreach \y in {0,1,2,3} { 
    \node[circle, fill = gray!20, inner sep=0pt, minimum size=.5cm, draw=black, label=center:${}$] (X4\y) at (\y-0.5, -1) {};
        \foreach \x in {0,1,2} { 
	    \draw[->,shorten >= 2pt, shorten <= 2pt] (X3\x) -- (X4\y);   }
    }
    \end{tikzpicture}
    \caption{By Proposition~\ref{prop:narrow2}, this network is not a universal approximator of $\D_{1,4}$, although it has more than $\dim(\Delta_{1,4})=32$ parameters.} 
    \label{fig:my_label1}
\end{figure}
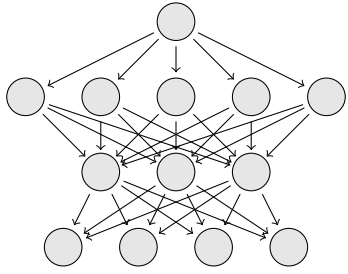

\section{A Numerical Example}
\label{sec:examples}

The previous theory is constructive in the sense that given a specific Markov kernel $P \in \D_{d,s}$, then for any choice of the shape coefficient $j$, the parameters are explicitly determined. 
To validate this for the $d = 2$, $s = 2$ case, $500$ Markov kernels in $\D_{2,2}$ were generated uniformly at random by sampling from the Dirichlet distribution with parameter $1$. 
For each kernel, a network consisting of the parameters specified by the theory above was instantiated.
We consider the architecture with $j=d=2$. 
As the magnitude of the non-zero parameters grows, the network will converge to the target kernel according to Theorem~\ref{thrm:errr}. 

Let $P^*$ be the target kernel and $P$ the approximation represented by the network (for the relatively small number of variables, it could be calculated exactly, but here we calculated it via 25,000 samples of the output for each input). 
The error is $E = \max_{i,j} |{P}_{ij} - P^*_{ij}|$. 
In the table below, we report the average error over 500 target kernels, $E_{\text{avg}} = \frac{1}{500} \sum_{k=1}^{500} E_k$, and the maximum error $E_{\text{max}} = \max_{k} E_k$, for the various values of the coefficient $\e$ from our theorem, along with the corresponding parameter magnitude bound $\alpha$, and the error upper bound of Theorem~\ref{thrm:errr} $|p({\bf y} | {\bf x}) - p^*({\bf y} | {\bf x})| \leq 1 - (1-\e)^N + 2\e$. 
\begin{center}
\begin{tabular}{| c | c | c | c | c |} 
            \hline
            $10\e$ & $\alpha$ & Error Bound of Thm.~\ref{thrm:errr} & $E_{\text{avg}}$ & $E_{\text{max}}$ \\
            \hline
            $2^{-2}$ & 14.65 & 0.4160 & 0.0522 & 0.1642\\ \hline
            $2^{-3}$ & 17.47 & 0.2276 & 0.0248 & 0.1004\\ \hline
            $2^{-4}$ & 20.28 & 0.1192 & 0.0134 & 0.0541\\ \hline
            $2^{-5}$ & 23.06 & 0.0610 & 0.0077 & 0.0425\\ \hline
            $2^{-6}$ & 25.84 & 0.0308 & 0.0060 & 0.0306\\ \hline
\end{tabular}
\end{center}

\begin{center}
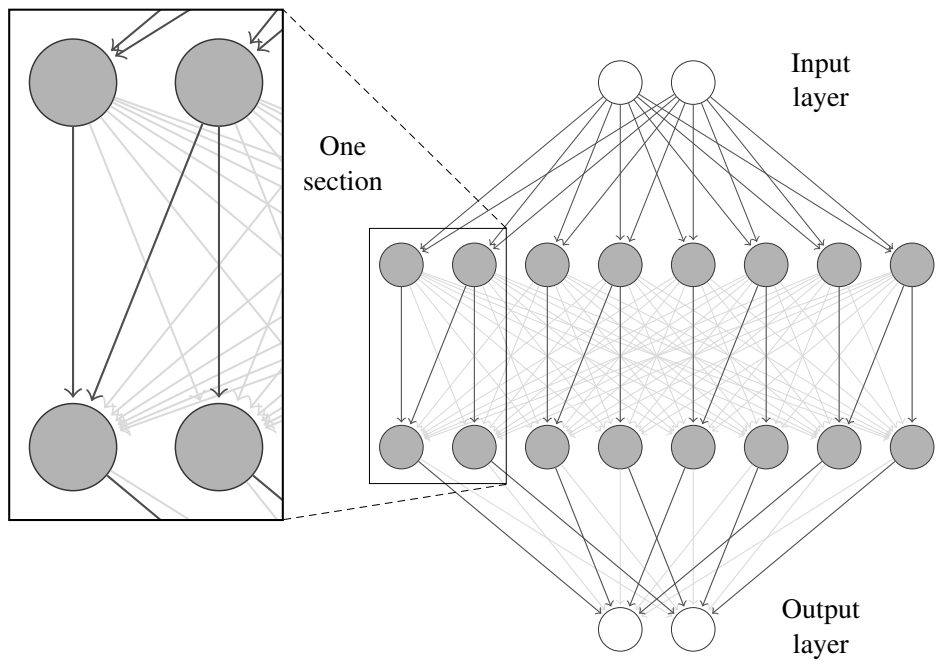
\begin{figure} 
\scalebox{0.97}{
\hspace{0.9cm}
\begin{tikzpicture}[shorten >=1pt,->,draw=black!50, node distance=\layersep, spy using outlines]
   
    \tikzstyle{every pin edge}=[<-,shorten <=1pt]
    \tikzstyle{neuron}=[circle,fill=black!25,minimum size=17pt,inner sep=0pt]
    \tikzstyle{input neuron}=[neuron, draw=black!80, line width=0.1mm, fill = black!0!white];
    \tikzstyle{output neuron}=[neuron, draw=black!80, line width=0.1mm, fill = black!0!white];
    \tikzstyle{hidden neuron}=[neuron, draw=black!80, line width=0.1mm, fill = black!30!white];
    \tikzstyle{annot} = [text width=4em, text centered]

    \foreach \name / \y in {1,...,2}
        \node[input neuron] (I-\name) at (-\y-3,7.5) {};

    \foreach \name / \y in {1,...,8}
        \path node[hidden neuron] (H-\name) at (-\y,5 cm) {};

    \foreach \name / \y in {1,...,8}
        \path node[hidden neuron] (H1-\name) at (-\y,2.5 cm) {};

    \foreach \name / \y in {1,...,2}
        \node[output neuron] (O-\name) at (-\y-3,0 cm) {};

    \foreach \source in {1,...,2}
        \foreach \dest in {1,...,8}
            \path (I-\source) edge[black!70] (H-\dest);

    \foreach \source in {1,...,8}
        \foreach \dest in {1,...,8}
            \path (H-\source) edge[black!15] (H1-\dest);

    \foreach \source in {1,...,8}
        \foreach \dest in {1,...,2}
            \path (H1-\source) edge[black!15] (O-\dest);

    \foreach \source in {1,...,8}
        \path (H-\source) edge[black!70] (H1-\source);

    \path (H-7) edge[black!70] (H1-8);
    \path (H-5) edge[black!70] (H1-6);
    \path (H-3) edge[black!70] (H1-4);
    \path (H-1) edge[black!70] (H1-2);

    \foreach \source in {8,6,4,2}
        \path (H1-\source) edge[black!70] (O-2);
    \foreach \source in {7,5,3,1}
        \path (H1-\source) edge[black!70] (O-1);

    \spy [draw, height=7cm, width=3.75cm, magnification=2] on (-7.5,3.75) in node at ($(H-7)!.5!(H-8) + (-4,0) $);

    \draw[-,densely dashed] (-6.56,5.5) -- ($(H-7)!.5!(H-8) + (-2.14,3.52) $);
    \draw[-,densely dashed] (-6.56,1.99) -- ($(H-7)!.5!(H-8) + (-2.15,-3.5) $);
    
    \node[annot] at ($(H-7)!.5!(H-8) + (-1.3,1.4) $) {One section};
    \node[annot,right of=I-1, node distance=1.75cm] {Input layer};
    \node[annot,right of=O-1, node distance=1.75cm] {Output layer};
\end{tikzpicture}}
\caption{The network architecture used in the numerical example of Section~\ref{sec:examples} to demonstrate our results.
The non-zero connections are shown in black. 
Notice that even within sections the connectivity does not need to be full. 
} 
\label{fig:examp1}
\end{figure}
\end{center}

\newpage
We now describe the explicit construction. 
Write a given kernel as 
\begin{align*} P &= \bpm 
p(0,0 | 0,0) & p(0,1 | 0,0) & p(1,0 | 0,0) & p(1,1 | 0,0) \\
p(0,0 | 0,1) & p(0,1 | 0,1) & p(1,0 | 0,1) & p(1,1 | 0,1) \\
p(0,0 | 1,0) & p(0,1 | 1,0) & p(1,0 | 1,0) & p(1,1 | 1,0) \\
p(0,0 | 1,1) & p(0,1 | 1,1) & p(1,0 | 1,1) & p(1,1 | 1,1) \epm \\
&= \bpm 
p_{11} & p_{12} & p_{13} & p_{14} \\ 
p_{21} & p_{22} & p_{23} & p_{24} \\ 
p_{31} & p_{32} & p_{33} & p_{34} \\ 
p_{41} & p_{42} & p_{43} & p_{44} \epm . 
\end{align*}
Given a fixed kernel of this form, the following choices of network parameters will make the network exactly approximate the kernel as we allow the maximum magnitude of the parameters $\a \to \infty$. 

We consider the widest deep architecture with $j=d$. 
Since we have $d=2$ inputs and $s = 2 = 2^{b-1}+b$ outputs, with $b=1$, 
the network has $2^{d-j}(\frac{2^s}{2(s-b)} +2(s-b)-1)-1 = 2$ 
hidden layers of width $2^j (s+d-j) = 8$. 
Here, since $j=d$, we can save one layer in comparison to the general construction.
This network consists of $2^j=4$ independent sections, one for each possible input. 
See Figure~\ref{fig:examp1} for an illustration of the overall network topology. 

The sharing schedule of the single subsection of each section will follow the $2^b=2(d-b)=2$ partial Gray codes 
\begin{equation}
 \sS_{1} = \bpm \sS_{1,1}\\\sS_{1,2}\epm = 
 \bpm 
 1 & 0\\ 
 0 & 0
 \epm 
 \quad 
 \text{and}
 \quad 
 \sS_{2} = 
 \bpm \sS_{2,1}\\\sS_{2,2}\epm = 
\bpm 
 1 & 1\\ 
 0 & 1
 \epm.  
\end{equation}
For a given $0<\e<1/4$ we set $\gamma = \sigma^{-1}(1-\e)$. 
The parameters will have magnitude at most $\alpha = 2m\sigma^{-1}(1-\e)$, where $m =\max\{j,s+(d-j)\}=2$.  
The weights and biases of the first hidden layer are  
\begin{equation} \label{eq:125}
{\bf W}^1 = 
\bpm 
-2\g & -2\g \\ 
-\g & -\g \\ 
\hline
-2\g & 2\g \\ 
-2\g & 2\g \\ 
\hline
2\g & -2\g \\ 
2\g & -2\g \\ 
\hline
2\g & 2\g \\ 
\g & \g \epm, 
\;\;\; 
{\bf b}^1 = \bpm 
\g \\ \s^{-1}(p_{12} + p_{14}) \\ \hline
-\g \\ \s^{-1}(p_{22} + p_{24}) -2\g\\\hline 
-\g \\ \s^{-1}(p_{32}+p_{34})-2\g \\\hline
-3\g\\ \s^{-1}(p_{41} + p_{44}) - 2\g \epm . 
\end{equation}
This will map ${\bf x}=(0,0)$ to ${\bf a}_{(1)}^1=(1,0)$ with probability $(1-\e)(1-(p_{12}+p_{14}))=(1-\e)(p_{11}+p_{13})$ and to ${\bf a}_{(1)}^1=(1,1)$ with probability $(1-\e)(p_{12}+p_{14})$. The other transitions are similar.
When $P \in \D^\e_{2,2}$, one necessarily has that $\e \leq 1/4$, otherwise $\D^\e_{2,2}$ is empty, $p_{ij} \leq 1-3\e$, and $p_{ij} + p_{ik} \leq 1-2\e$ for any $i,j,k$.  This in turn means the bias parameters in (\ref{eq:125}) are bounded by $3\g$.

The second hidden layer has weights 
\begin{equation} {\bf W}^2 = \bpm 
\g & \o_1 & 0 & 0 & 0 & 0 & 0 & 0\\
0 & 2\g & 0 & 0 & 0 & 0 & 0 & 0\\\hline
0 & 0  & \g & \o_2 & 0 & 0 & 0 & 0\\
0 & 0 & 0 & 2\g & 0 & 0 & 0 & 0\\\hline
0 & 0 & 0 & 0 & \g & \o_3 & 0 & 0\\
0 & 0 & 0 & 0 & 0 & 2\g & 0 & 0\\\hline
0 & 0 & 0 & 0 & 0 & 0 & \g& \o_4\\
0 & 0 & 0 & 0 & 0 & 0 & 0 & 2\g
\epm , 
\end{equation}
where 
\begin{equation} \label{eqn:127}
\o_i = \s^{-1}\left(\frac{p_{i4}}{p_{i2} + p_{i4}}\right) - \s^{-1}\left(\frac{p_{i3}}{p_{i1}+p_{i3}}\right) , 
\end{equation}
for $i = 1,2,3,4$. 
The second hidden layer biases are chosen as 
\begin{equation} \label{eqn:128}
{\bf b}^2 = \bpm 
\s^{-1}(p_{13}/(p_{11}+p_{13})) -\g\\ -\g \\\hline 
\s^{-1}(p_{23}/(p_{21} + p_{23})) -\g\\ -\g \\\hline 
\s^{-1}(p_{33}/(p_{33}+p_{31}))-\g\\ -\g \\\hline 
\s^{-1}(p_{43}/(p_{43}+p_{41})) -\g\\ -\g \epm . 
\end{equation}

This will map ${\bf a}_{(1)}^1 = (1,0)$ to ${\bf a}_{(1)}^2 = (1,0)$ with probability $p_{13}/(p_{11}+p_{13})(1-\e)$ and to ${\bf a}_{(1)}^2 = (0,0)$ with probability $(1-p_{13}/(p_{11}+p_{13}))(1-\e)$. 
In particular, for input ${\bf x}=(0,0)$ we have 
\begin{align*}
\Pr\big({\bf a}_{(1)}^2=(0,0),{\bf a}_{(1)}^1=(1,0) | {\bf x}=(0,0)\big) 
&= \big(1-\frac{p_{13}}{p_{11}+p_{13}}\big)(1-\e) \cdot (1-\e)(p_{11}+p_{13}) \\
&= p_{11}(1-\e)^2. 
\end{align*}
Note that if $\e\leq p_{ij}\leq 1-\e$ for all $i,j$, then the factors $p_{13}/(p_{11}+p_{13})$ and $(p_{11}+p_{13})$ are also between $\e$ and $1-\e$. 
The other transitions are similar. 

The weights and biases of the output layer are 
\begin{equation}
{\bf W}^3 = \bpm 2\g & 0 & 2\g & 0 & 2\g & 0 & 2\g & 0 \\ 0 & 2\g & 0 & 2\g & 0 & 2\g & 0 & 2\g\epm, \;\;\; {\bf b}^3 = \bpm -\g \\ -\g \epm . 
\end{equation}
This will map, for example, ${\bf h}^2=({\bf a}_{(1)}^2,{\bf a}_{(2)}^2,{\bf a}_{(3)}^2,{\bf a}_{(4)}^2)=(0,0,0,\ldots,0)$ to ${\bf y} = (0,0)$ with probability $(1-\sigma(-\gamma))^2 = (1-\e)^2$. 
Furthermore, to see that the parameters in \eqref{eqn:127} and \eqref{eqn:128} are bounded by $2\g$, it needs to be shown that $\frac{p_{ij}}{p_{ij} + p_{ik}} \leq 1-\e$, or that $1 + \frac{p_{ik}}{p_{ij}} \geq (1-\e)^{-1}$.  The smallest possible  $\frac{p_{ik}}{p_{ij}}$ is $\frac{\e}{1-3\e}$, and $\frac{1}{1-\e} \leq 1 + \frac{\e}{1-3\e}$ is easily verified for $\e \in (0,1/4]$.

Note that exactly $2^2(2^2-1) = 12$ parameters depend on the values of the target Markov kernel itself, while the other parameters are fixed depending only on the desired level of accuracy.

\section{Conclusion} 
\label{sec:conclusion}

In this work we made advances towards a more complete picture of the representational power of stochastic feedforward networks. 
We showed that a spectrum of sigmoid stochastic feedforward networks are capable of universal approximation. 
In the obtained results, the shallow architecture requires less hidden units than the deep architectures, while the deep architectures achieve the minimum number of trainable parameters necessary for universal approximation. 
At the extreme of the spectrum discussed is the $j = 0$ case, where a network of width $s+d$ and depth approx $2^{d+s}/2(s-b)$, $b \sim \log_2(s)$ is sufficient for universal approximation. 
At the other end of the deep spectrum is the $j = d$ case, which can be seen as an intermediate between the $j = 0$ case and the shallow universal approximator since its width is exponential in $d$ and its depth is exponential in~$s$. 
Further, we obtained bounds on the approximation errors when the network parameters are restricted in absolute value by some $\alpha>0$. In our construction, the error is then bounded by $1-(1-\e)^N+2\e$, where $\e$ is the error of each unit, which is bounded by $\sigma(-\alpha/2(d+s))$.

\subsection*{Open Problems}

We collect a few open problems that we believe would be worth developing in the future. 

\begin{itemize}

\item Determine the dimension of the set of distributions represented at layer $l$ of a deep stochastic network. In this direction, the dimension of RBMs has been studied \citep{cueto2010geometry,doi:10.1137/16M1077489}. 

\item Determine the equations and inequalities that characterize the set of Markov kernels that are representable by a deep stochastic feedforward network. 
Works in this direction include studies of Bayesian networks \citep{GARCIA2005331}, 
the RBM with three visible and two hidden binary variables \citep{seigal2018mixtures}, 
and naive Bayes models with one hidden binary variable \citep{allman2019maximum}. 

\item Obtain maximum approximation error bounds for a network which is \textit{not} a universal approximator by measures such as the maximum KL-divergence, and evaluate the behavior of different network topologies depending on the number of trainable parameters. 
There are a number of works in this direction, covering hierarchical graphical models \citep{matus2009divergence}, exponential families \citep{rauh2011finding}, RBMs and DBNs \citep{le2008representational,montufar2011expressive,montufar2013maximal,montufar2015geometry}. 
For the RBM with three visible and two hidden units mentioned above, \citet{seigal2018mixtures} obtained the exact value. 

\item Is it possible to obtain more compact families of universal approximators of Markov kernels than the ones that we presented here? 
We constructed universal approximators with the minimal number of trainable weights, but which include a substantial number of non-zero fixed weights. Is it possible to construct more compact universal approximators with a smaller number of units and non-zero weights? 
Can we refine the lower bounds for the minimum width and the minimum depth given a maximum width of the hidden layers of a universal approximator? 
This kind of problem has traditionally been more difficult than refining upper bounds. 
A few examples are listed by \citet{montufar2017hierarchical}. 

\item Our construction uses sparsely connected networks to achieve the minimum possible number of parameters. How does restricting the connectivity of a network affect the distributions it can represent? Describe whether and which advantages are provided by sparsely connected networks over fully connected networks. 

\item Generalize the analysis to conditional distributions other than sigmoid, and non-binary variables. Works in this direction inlcude the treatment of RBMs with nonbinary units \citep{JMLR:v16:montufar15a} and that of DBNs with non-binary units \citep{montufar2014universal}. 

\item Another interesting direction are the theoretical advantages of stochastic networks in relation to deterministic networks, and the development of more effective techniques for training stochastic networks. In this direction, \citet{NIPS2013_5026} discuss multi-modality and combinations of deterministic and stochastic units. 

\end{itemize}

\bigskip 
\noindent\textbf{Acknowledgement}
\small 
We thank Nihat Ay for insightful discussions. This project has received funding from the European Research Council (ERC) under the European Union's Horizon 2020 research and innovation programme (grant agreement n\textsuperscript{o} 757983).

\markboth{Stochastic Feedforward Neural Networks:\\Universal Approximation}{\it{References}}
\bibliography{References}
\bibliographystyle{plainnat}

\end{document}